\documentclass[nosubfloats]{hld2024} 

\usepackage[utf8]{inputenc}
\usepackage[T1]{fontenc}
\usepackage{microtype}

\usepackage[most]{tcolorbox}
\usepackage{amsfonts,amssymb,mathrsfs,bbm,mathtools,nicefrac,bm,centernot,derivative}
\usepackage{lipsum}
\usepackage{cleveref}
\usepackage{enumitem}
\usepackage{wrapfig}
\usepackage{subcaption}
\usepackage{xspace}
\usepackage{verbatim}
\usepackage{hyperref}
\usepackage{booktabs}
\usepackage{multirow}
\usepackage{siunitx}
\usepackage{pifont}
\usepackage{tikz}
\usepackage{pgfplots}
\usepackage{caption}
\usepackage[page,header]{appendix}
\usepackage{titletoc}

\usepackage[english]{babel}
\usepackage[autolanguage]{numprint}
\nprounddigits{3}

\AddToHook{env/lemma/begin}{\crefalias{theorem}{lemma}}
\crefname{lemma}{lemma}{lemmas}

\newcommand{\DIT}[2]{\acs{dit}-#1/#2}

\newcommand{\MP}[1]{\ensuremath{\mathcal{M}\lft[#1\rgt]}}
\newcommand{\cmark}{\ding{51}}
\newcommand{\xmark}{\ding{55}}

\newcommand{\ema}{\hat{\bm{\theta}}_\gamma}
\newcommand{\emaG}[1]{\hat{\bm{\theta}}_{#1}}
\newcommand{\var}{\bm{\theta}}
\newcommand{\sigrel}{\ensuremath{\sigma_{\mathrm{rel}}}}

\newcommand{\lft}{\mathopen{}\mathclose\bgroup\left}
\newcommand{\rgt}{\aftergroup\egroup\right}

\DeclarePairedDelimiterX{\infdivx}[2]{(}{)}{%
  #1\;\delimsize\|\;#2%
}

\newcommand{\R}{\mathbb{R}}

\newcommand{\E}{\mathbb{E}}

\renewcommand{\vec}[1]{\mathbf{#1}}
\newcommand{\mat}[1]{\mathbf{#1}}

\newcommand{\trp}[1]{#1^\top}
\renewcommand{\det}[1]{\operatorname{det}\lft( #1 \rgt)}

\newcommand{\defeq}{\triangleq}

\makeatletter
\DeclareRobustCommand\onedot{\futurelet\@let@token\@onedot}
\def\@onedot{\ifx\@let@token.\else.\null\fi\xspace}

\def\ie{\emph{i.e}\onedot}

\def\iid{\emph{i.i.d}\onedot} 
\makeatother

\pgfplotsset{
    compat=1.17,
    compat/show suggested version=false,
    legend image code/.code={
        \draw[mark indices={2}] plot coordinates {
            (0cm,0cm) (0.15cm,0cm) (0.3cm,0cm)
        };
    }
}
\usepgfplotslibrary{fillbetween}

\definecolor{C0}{HTML}{0090FF}
\definecolor{C1}{HTML}{E54D2E}
\definecolor{C2}{HTML}{46A758}
\definecolor{C3}{HTML}{F76B15}
\definecolor{C4}{HTML}{D6409F}
\definecolor{C5}{HTML}{AD7F58}

\NewDocumentCommand\AddMinimumPlot{O{black}O{below}mm}{%
    \addplot [
        name path=lower,
        thick,
        mark=*,
        mark size={2.5pt},
        mark indices={\ldata{0}{#4-argmin}},
        nodes near coords align={#2},
        #1
        ] coordinates { \rawdata{#4} };

    \ifblank{#3}{}{\addlegendentry{#3};}
}

\NewDocumentCommand\AddMaximumPlot{O{black}O{above}mm}{%
    \addplot [
        name path=lower,
        thick,
        mark=*,
        mark size={2.5pt},
        mark indices={\ldata{0}{#4-argmax}},
        nodes near coords align={#2},
        #1
        ] coordinates { \rawdata{#4} };

    \ifblank{#3}{}{\addlegendentry{#3};}
}

\usepackage{acro}
\DeclareAcronym{dit}{
    short=DiT,
    long=Diffusion Transformer,
    cite=peebles_scalable_2023,
}

\DeclareAcronym{vit}{
    short=ViT,
    long=Vision Transformer,
    cite=peebles_scalable_2023,
}

\DeclareAcronym{mapdit}{
    short=MaP-DiT,
    long=magnitude-preserving Diffusion Transformer,
}

\DeclareAcronym{fid}{
    short=FID,
    long=Fréchet inception distance,
    cite=heusel_gans_2018,
}

\DeclareAcronym{sfid}{
    short=sFID,
    long=spatial FID,
    cite=nash_generating_2021,
}

\DeclareAcronym{ema}{
    short=EMA,
    long=exponential moving average,
}

\DeclareAcronym{adaln}{
    short=AdaLN,
    long=Adaptive Layer Normalization,
}

\DeclareAcronym{vae}{
    short=VAE,
    long=variational autoencoder,
    cite=kingma2013auto,
}

\DeclareAcronym{silu}{
    short=SiLU,
    long=sigmoid linear unit,
    cite=elfwing2018sigmoid,
}

\newcommand{\defdata}[2]{\expandafter\newcommand\csname data-#1\endcsname{#2}}

\newcommand{\ldata}[2]{\ifdata{#2}\rawdata{#2}\else#1\fi}
\newcommand{\ifdata}[1]{\ifcsname data-#1\endcsname}
\newcommand{\rawdata}[1]{\csname data-#1\endcsname}

\defdata{e_xs_2-guidance5.3---ema-fid10k}{
(2.5,  66.2184)
(5.0,  34.0969)
(7.5,  34.7270)
(10.0, 34.0980)
(12.5, 33.5370)
(15.0, 33.4048)
(17.5, 33.5279)
(20.0, 33.6268)
(22.5, 34.2573)
(25.0, 35.7749)
}
\defdata{e_xs_2-guidance5.3---ema-fid10k-argmin}{5}

\defdata{e_xs_2-guidance5.4---ema-fid10k}{
(2.5,  65.7314)
(5.0,  34.0684)
(7.5,  34.7662)
(10.0, 34.0659)
(12.5, 33.5856)
(15.0, 33.4219)
(17.5, 33.5401)
(20.0, 33.6360)
(22.5, 34.3179)
(25.0, 35.8028)
}
\defdata{e_xs_2-guidance5.4---ema-fid10k-argmin}{5}

\defdata{e_xs_2-guidance5.5---ema-fid10k}{
(2.5,  65.3559)
(5.0,  34.0324)
(7.5,  34.7871)
(10.0, 34.0343)
(12.5, 33.6641)
(15.0, 33.4841)
(17.5, 33.5319)
(20.0, 33.7012)
(22.5, 34.3351)
(25.0, 35.7684)
}
\defdata{e_xs_2-guidance5.5---ema-fid10k-argmin}{5}

\defdata{e_xs_2-guidance5.6---ema-fid10k}{
(2.5,  64.9932)
(5.0,  34.0658)
(7.5,  34.8593)
(10.0, 34.0640)
(12.5, 33.6596)
(15.0, 33.5342)
(17.5, 33.5489)
(20.0, 33.7067)
(22.5, 34.3409)
(25.0, 35.7666)
}
\defdata{e_xs_2-guidance5.6---ema-fid10k-argmin}{5}

\defdata{e_xs_2-guidance5.7---ema-fid10k}{
(2.5,  64.6577)
(5.0,  34.0851)
(7.5,  34.8756)
(10.0, 34.0810)
(12.5, 33.7141)
(15.0, 33.5434)
(17.5, 33.5730)
(20.0, 33.7051)
(22.5, 34.3320)
(25.0, 35.8335)
}
\defdata{e_xs_2-guidance5.7---ema-fid10k-argmin}{5}
\defdata{a_xs_2-noguid---steps-is}{
(50000,  7.9590)
(100000, 8.4172)
(150000, 9.1370)
(200000, 9.5503)
(250000, 9.6497)
(300000, 9.6142)
(350000, 9.4520)
(400000, 10.1936)
}
\defdata{a_xs_2-noguid---steps-is-argmax}{8}
\defdata{a_xs_2-noguid---steps-is-max}{10.1936}

\defdata{a_xs_2-noguid---steps-fid10k}{
(50000,  118.1709)
(100000, 111.8548)
(150000, 109.1962)
(200000, 105.9967)
(250000, 103.6092)
(300000, 105.4651)
(350000, 104.0987)
(400000, 99.4253)
}
\defdata{a_xs_2-noguid---steps-fid10k-argmin}{8}
\defdata{a_xs_2-noguid---steps-fid10k-min}{99.4253}

\defdata{a_xs_2-noguid---steps-sfid10k}{
(50000,  42.6975)
(100000, 40.3952)
(150000, 37.8384)
(200000, 35.7831)
(250000, 34.9783)
(300000, 35.3895)
(350000, 36.6488)
(400000, 32.6109)
}
\defdata{a_xs_2-noguid---steps-sfid10k-argmin}{8}
\defdata{a_xs_2-noguid---steps-sfid10k-min}{32.6109}

\defdata{a_xs_2-noguid---steps-precision}{
(50000,  0.1629)
(100000, 0.1680)
(150000, 0.1607)
(200000, 0.1548)
(250000, 0.1640)
(300000, 0.1660)
(350000, 0.1861)
(400000, 0.1762)
}
\defdata{a_xs_2-noguid---steps-precision-argmax}{7}
\defdata{a_xs_2-noguid---steps-precision-max}{0.1861}

\defdata{a_xs_2-noguid---steps-recall}{
(50000,  0.1676)
(100000, 0.1881)
(150000, 0.2003)
(200000, 0.2262)
(250000, 0.2377)
(300000, 0.2346)
(350000, 0.2351)
(400000, 0.2475)
}
\defdata{a_xs_2-noguid---steps-recall-argmax}{8}
\defdata{a_xs_2-noguid---steps-recall-max}{0.2475}

\defdata{a_xs_2-guid---steps-is}{
(50000,  16.1824)
(100000, 21.3134)
(150000, 24.7887)
(200000, 27.3606)
(250000, 29.7212)
(300000, 31.6440)
(350000, 33.5160)
(400000, 34.9841)
}
\defdata{a_xs_2-guid---steps-is-argmax}{8}
\defdata{a_xs_2-guid---steps-is-max}{34.9841}

\defdata{a_xs_2-guid---steps-fid10k}{
(50000,  70.0440)
(100000, 58.1026)
(150000, 52.9261)
(200000, 50.3205)
(250000, 47.7453)
(300000, 45.9116)
(350000, 43.9160)
(400000, 43.4003)
}
\defdata{a_xs_2-guid---steps-fid10k-argmin}{8}
\defdata{a_xs_2-guid---steps-fid10k-min}{43.4003}

\defdata{a_xs_2-guid---steps-sfid10k}{
(50000,  39.8665)
(100000, 36.1646)
(150000, 35.8813)
(200000, 35.2033)
(250000, 36.2479)
(300000, 34.9811)
(350000, 31.9805)
(400000, 34.8469)
}
\defdata{a_xs_2-guid---steps-sfid10k-argmin}{7}
\defdata{a_xs_2-guid---steps-sfid10k-min}{31.9805}

\defdata{a_xs_2-guid---steps-precision}{
(50000,  0.2385)
(100000, 0.3049)
(150000, 0.3418)
(200000, 0.3504)
(250000, 0.3688)
(300000, 0.3920)
(350000, 0.4108)
(400000, 0.4184)
}
\defdata{a_xs_2-guid---steps-precision-argmax}{8}
\defdata{a_xs_2-guid---steps-precision-max}{0.4184}

\defdata{a_xs_2-guid---steps-recall}{
(50000,  0.1692)
(100000, 0.2125)
(150000, 0.2314)
(200000, 0.2460)
(250000, 0.2470)
(300000, 0.2539)
(350000, 0.2617)
(400000, 0.2666)
}
\defdata{a_xs_2-guid---steps-recall-argmax}{8}
\defdata{a_xs_2-guid---steps-recall-max}{0.2666}

\defdata{b_xs_2-noguid---steps-is}{
(50000,  7.9026)
(100000, 8.5685)
(150000, 9.2460)
(200000, 9.5620)
(250000, 9.4996)
(300000, 10.0063)
(350000, 10.1616)
(400000, 10.0989)
}
\defdata{b_xs_2-noguid---steps-is-argmax}{7}
\defdata{b_xs_2-noguid---steps-is-max}{10.1616}

\defdata{b_xs_2-noguid---steps-fid10k}{
(50000,  118.5797)
(100000, 114.2455)
(150000, 105.1693)
(200000, 102.6365)
(250000, 105.8833)
(300000, 101.0600)
(350000, 101.4651)
(400000, 99.1432)
}
\defdata{b_xs_2-noguid---steps-fid10k-argmin}{8}
\defdata{b_xs_2-noguid---steps-fid10k-min}{99.1432}

\defdata{b_xs_2-noguid---steps-sfid10k}{
(50000,  42.0106)
(100000, 40.7550)
(150000, 35.9017)
(200000, 34.7508)
(250000, 37.9000)
(300000, 33.4346)
(350000, 34.6750)
(400000, 33.4787)
}
\defdata{b_xs_2-noguid---steps-sfid10k-argmin}{6}
\defdata{b_xs_2-noguid---steps-sfid10k-min}{33.4346}

\defdata{b_xs_2-noguid---steps-precision}{
(50000,  0.1543)
(100000, 0.1545)
(150000, 0.1672)
(200000, 0.1620)
(250000, 0.1724)
(300000, 0.1706)
(350000, 0.1697)
(400000, 0.1740)
}
\defdata{b_xs_2-noguid---steps-precision-argmax}{8}
\defdata{b_xs_2-noguid---steps-precision-max}{0.1740}

\defdata{b_xs_2-noguid---steps-recall}{
(50000,  0.1695)
(100000, 0.2100)
(150000, 0.2171)
(200000, 0.2139)
(250000, 0.2477)
(300000, 0.2294)
(350000, 0.2530)
(400000, 0.2576)
}
\defdata{b_xs_2-noguid---steps-recall-argmax}{8}
\defdata{b_xs_2-noguid---steps-recall-max}{0.2576}

\defdata{b_xs_2-guid---steps-is}{
(50000,  16.7064)
(100000, 21.9171)
(150000, 25.8924)
(200000, 28.2173)
(250000, 31.1013)
(300000, 33.4139)
(350000, 34.7086)
(400000, 37.1100)
}
\defdata{b_xs_2-guid---steps-is-argmax}{8}
\defdata{b_xs_2-guid---steps-is-max}{37.1100}

\defdata{b_xs_2-guid---steps-fid10k}{
(50000,  68.9698)
(100000, 57.1789)
(150000, 51.6049)
(200000, 49.6514)
(250000, 46.5396)
(300000, 44.4236)
(350000, 43.4035)
(400000, 42.1489)
}
\defdata{b_xs_2-guid---steps-fid10k-argmin}{8}
\defdata{b_xs_2-guid---steps-fid10k-min}{42.1489}

\defdata{b_xs_2-guid---steps-sfid10k}{
(50000,  38.3833)
(100000, 37.5017)
(150000, 36.5946)
(200000, 38.8265)
(250000, 34.3426)
(300000, 36.8543)
(350000, 34.2371)
(400000, 35.2443)
}
\defdata{b_xs_2-guid---steps-sfid10k-argmin}{7}
\defdata{b_xs_2-guid---steps-sfid10k-min}{34.2371}

\defdata{b_xs_2-guid---steps-precision}{
(50000,  0.2319)
(100000, 0.2944)
(150000, 0.3311)
(200000, 0.3503)
(250000, 0.3872)
(300000, 0.4057)
(350000, 0.4135)
(400000, 0.4335)
}
\defdata{b_xs_2-guid---steps-precision-argmax}{8}
\defdata{b_xs_2-guid---steps-precision-max}{0.4335}

\defdata{b_xs_2-guid---steps-recall}{
(50000,  0.1704)
(100000, 0.2116)
(150000, 0.2327)
(200000, 0.2335)
(250000, 0.2502)
(300000, 0.2472)
(350000, 0.2618)
(400000, 0.2724)
}
\defdata{b_xs_2-guid---steps-recall-argmax}{8}
\defdata{b_xs_2-guid---steps-recall-max}{0.2724}

\defdata{c_xs_2-noguid---steps-is}{
(50000,  9.5533)
(100000, 9.7044)
(150000, 11.0412)
(200000, 11.1809)
(250000, 11.0283)
(300000, 11.9368)
(350000, 11.5283)
(400000, 11.6833)
}
\defdata{c_xs_2-noguid---steps-is-argmax}{6}
\defdata{c_xs_2-noguid---steps-is-max}{11.9368}

\defdata{c_xs_2-noguid---steps-fid10k}{
(50000,  105.3583)
(100000, 103.5105)
(150000, 93.5007)
(200000, 92.9462)
(250000, 93.6506)
(300000, 86.4876)
(350000, 88.0029)
(400000, 87.4270)
}
\defdata{c_xs_2-noguid---steps-fid10k-argmin}{6}
\defdata{c_xs_2-noguid---steps-fid10k-min}{86.0029}

\defdata{c_xs_2-noguid---steps-sfid10k}{
(50000,  35.7222)
(100000, 36.0370)
(150000, 31.5667)
(200000, 30.3647)
(250000, 31.7334)
(300000, 27.8911)
(350000, 28.8597)
(400000, 29.0145)
}
\defdata{c_xs_2-noguid---steps-sfid10k-argmin}{6}
\defdata{c_xs_2-noguid---steps-sfid10k-min}{27.8911}

\defdata{c_xs_2-noguid---steps-precision}{
(50000,  0.1664)
(100000, 0.1786)
(150000, 0.1761)
(200000, 0.1791)
(250000, 0.1871)
(300000, 0.1844)
(350000, 0.1936)
(400000, 0.1947)
}
\defdata{c_xs_2-noguid---steps-precision-argmax}{8}
\defdata{c_xs_2-noguid---steps-precision-max}{0.1947}

\defdata{c_xs_2-noguid---steps-recall}{
(50000,  0.2366)
(100000, 0.2534)
(150000, 0.2953)
(200000, 0.2950)
(250000, 0.2975)
(300000, 0.3130)
(350000, 0.2992)
(400000, 0.3092)
}
\defdata{c_xs_2-noguid---steps-recall-argmax}{6}
\defdata{c_xs_2-noguid---steps-recall-max}{0.3130}

\defdata{c_xs_2-guid---steps-is}{
(50000,  24.2622)
(100000, 31.6440)
(150000, 36.2830)
(200000, 41.0598)
(250000, 42.6040)
(300000, 44.5520)
(350000, 46.6410)
(400000, 47.8427)
}
\defdata{c_xs_2-guid---steps-is-argmax}{8}
\defdata{c_xs_2-guid---steps-is-max}{47.8427}

\defdata{c_xs_2-guid---steps-fid10k}{
(50000,  52.8716)
(100000, 43.9287)
(150000, 41.3012)
(200000, 38.4256)
(250000, 37.3426)
(300000, 35.9199)
(350000, 34.6892)
(400000, 35.0723)
}
\defdata{c_xs_2-guid---steps-fid10k-argmin}{7}
\defdata{c_xs_2-guid---steps-fid10k-min}{34.6892}

\defdata{c_xs_2-guid---steps-sfid10k}{
(50000,  31.7875)
(100000, 31.0249)
(150000, 32.5576)
(200000, 34.5415)
(250000, 31.6762)
(300000, 33.8084)
(350000, 31.7536)
(400000, 32.7279)
}
\defdata{c_xs_2-guid---steps-sfid10k-argmin}{2}
\defdata{c_xs_2-guid---steps-sfid10k-min}{31.0249}

\defdata{c_xs_2-guid---steps-precision}{
(50000,  0.3304)
(100000, 0.4130)
(150000, 0.4440)
(200000, 0.4581)
(250000, 0.4934)
(300000, 0.5105)
(350000, 0.5328)
(400000, 0.5218)
}
\defdata{c_xs_2-guid---steps-precision-argmax}{7}
\defdata{c_xs_2-guid---steps-precision-max}{0.528}

\defdata{c_xs_2-guid---steps-recall}{
(50000,  0.2303)
(100000, 0.2595)
(150000, 0.2680)
(200000, 0.2836)
(250000, 0.2589)
(300000, 0.2850)
(350000, 0.2801)
(400000, 0.2844)
}
\defdata{c_xs_2-guid---steps-recall-argmax}{6}
\defdata{c_xs_2-guid---steps-recall-max}{0.2850}

\defdata{d_xs_2-noguid---steps-is}{
(50000,  9.8415)
(100000, 10.1595)
(150000, 10.3031)
(200000, 11.1602)
(250000, 11.1382)
(300000, 11.2203)
(350000, 11.4053)
(400000, 11.4792)
}
\defdata{d_xs_2-noguid---steps-is-argmax}{8}
\defdata{d_xs_2-noguid---steps-is-max}{11.4792}

\defdata{d_xs_2-noguid---steps-fid10k}{
(50000,  104.1882)
(100000, 98.5073)
(150000, 100.9525)
(200000, 93.3535)
(250000, 93.3474)
(300000, 92.2916)
(350000, 89.3819)
(400000, 90.4693)
}
\defdata{d_xs_2-noguid---steps-fid10k-argmin}{7}
\defdata{d_xs_2-noguid---steps-fid10k-min}{89.3819}

\defdata{d_xs_2-noguid---steps-sfid10k}{
(50000,  32.1339)
(100000, 33.4026)
(150000, 36.4275)
(200000, 31.8320)
(250000, 32.1429)
(300000, 31.0620)
(350000, 29.8257)
(400000, 31.4406)
}
\defdata{d_xs_2-noguid---steps-sfid10k-argmin}{7}
\defdata{d_xs_2-noguid---steps-sfid10k-min}{29.8257}

\defdata{d_xs_2-noguid---steps-precision}{
(50000,  0.1536)
(100000, 0.1699)
(150000, 0.1709)
(200000, 0.1847)
(250000, 0.1909)
(300000, 0.1871)
(350000, 0.2062)
(400000, 0.1903)
}
\defdata{d_xs_2-noguid---steps-precision-argmax}{7}
\defdata{d_xs_2-noguid---steps-precision-max}{0.2062}

\defdata{d_xs_2-noguid---steps-recall}{
(50000,  0.2128)
(100000, 0.2636)
(150000, 0.2399)
(200000, 0.2744)
(250000, 0.2696)
(300000, 0.2806)
(350000, 0.2892)
(400000, 0.3192)
}
\defdata{d_xs_2-noguid---steps-recall-argmax}{8}
\defdata{d_xs_2-noguid---steps-recall-max}{0.3192}

\defdata{d_xs_2-guid---steps-is}{
(50000,  22.9054)
(100000, 30.1448)
(150000, 32.1242)
(200000, 35.8563)
(250000, 40.5659)
(300000, 42.0974)
(350000, 45.7282)
(400000, 46.8632)
}
\defdata{d_xs_2-guid---steps-is-argmax}{8}
\defdata{d_xs_2-guid---steps-is-max}{46.8632}

\defdata{d_xs_2-guid---steps-fid10k}{
(50000,  54.9761)
(100000, 46.2379)
(150000, 44.0435)
(200000, 42.1411)
(250000, 37.4474)
(300000, 37.5230)
(350000, 35.0058)
(400000, 35.2128)
}
\defdata{d_xs_2-guid---steps-fid10k-argmin}{7}
\defdata{d_xs_2-guid---steps-fid10k-min}{35.0058}

\defdata{d_xs_2-guid---steps-sfid10k}{
(50000,  36.5192)
(100000, 34.1716)
(150000, 34.0931)
(200000, 38.7543)
(250000, 32.4077)
(300000, 35.3823)
(350000, 33.6654)
(400000, 32.7196)
}
\defdata{d_xs_2-guid---steps-sfid10k-argmin}{5}
\defdata{d_xs_2-guid---steps-sfid10k-min}{32.4077}

\defdata{d_xs_2-guid---steps-precision}{
(50000,  0.3158)
(100000, 0.3979)
(150000, 0.4114)
(200000, 0.4637)
(250000, 0.4979)
(300000, 0.4947)
(350000, 0.5259)
(400000, 0.5405)
}
\defdata{d_xs_2-guid---steps-precision-argmax}{8}
\defdata{d_xs_2-guid---steps-precision-max}{0.5405}

\defdata{d_xs_2-guid---steps-recall}{
(50000,  0.2361)
(100000, 0.2501)
(150000, 0.2663)
(200000, 0.2630)
(250000, 0.2886)
(300000, 0.3021)
(350000, 0.2900)
(400000, 0.2752)
}
\defdata{d_xs_2-guid---steps-recall-argmax}{6}
\defdata{d_xs_2-guid---steps-recall-max}{0.3021}

\defdata{e_xs_2-noguid---steps-is}{
(50000,  9.2608)
(100000, 9.2448)
(150000, 10.6046)
(200000, 11.0215)
(250000, 10.7232)
(300000, 11.2674)
(350000, 11.4399)
(400000, 11.7207)
}
\defdata{e_xs_2-noguid---steps-is-argmax}{8}
\defdata{e_xs_2-noguid---steps-is-max}{11.7207}

\defdata{e_xs_2-noguid---steps-fid10k}{
(50000,  110.5737)
(100000, 108.8893)
(150000, 95.5122)
(200000, 95.0614)
(250000, 96.7346)
(300000, 90.5799)
(350000, 86.7177)
(400000, 87.6193)
}
\defdata{e_xs_2-noguid---steps-fid10k-argmin}{7}
\defdata{e_xs_2-noguid---steps-fid10k-min}{86.7177}

\defdata{e_xs_2-noguid---steps-sfid10k}{
(50000,  38.6227)
(100000, 38.5814)
(150000, 31.8916)
(200000, 30.6043)
(250000, 32.9723)
(300000, 31.0442)
(350000, 29.7617)
(400000, 28.0122)
}
\defdata{e_xs_2-noguid---steps-sfid10k-argmin}{8}
\defdata{e_xs_2-noguid---steps-sfid10k-min}{28.0122}

\defdata{e_xs_2-noguid---steps-precision}{
(50000,  0.1449)
(100000, 0.1741)
(150000, 0.1755)
(200000, 0.1755)
(250000, 0.1797)
(300000, 0.1926)
(350000, 0.2010)
(400000, 0.1836)
}
\defdata{e_xs_2-noguid---steps-precision-argmax}{7}
\defdata{e_xs_2-noguid---steps-precision-max}{0.2010}

\defdata{e_xs_2-noguid---steps-recall}{
(50000,  0.2030)
(100000, 0.2405)
(150000, 0.2587)
(200000, 0.3101)
(250000, 0.2600)
(300000, 0.2906)
(350000, 0.2978)
(400000, 0.3227)
}
\defdata{e_xs_2-noguid---steps-recall-argmax}{8}
\defdata{e_xs_2-noguid---steps-recall-max}{0.3227}

\defdata{e_xs_2-guid---steps-is}{
(50000,  22.8276)
(100000, 29.7058)
(150000, 35.5566)
(200000, 38.8821)
(250000, 40.8012)
(300000, 44.1791)
(350000, 46.0895)
(400000, 47.7629)
}
\defdata{e_xs_2-guid---steps-is-argmax}{8}
\defdata{e_xs_2-guid---steps-is-max}{47.7629}

\defdata{e_xs_2-guid---steps-fid10k}{
(50000,  53.8236)
(100000, 45.3942)
(150000, 41.7330)
(200000, 38.3812)
(250000, 37.4696)
(300000, 35.7018)
(350000, 34.5944)
(400000, 34.0874)
}
\defdata{e_xs_2-guid---steps-fid10k-argmin}{8}
\defdata{e_xs_2-guid---steps-fid10k-min}{34.0874}

\defdata{e_xs_2-guid---steps-sfid10k}{
(50000,  32.9022)
(100000, 30.5378)
(150000, 33.6534)
(200000, 33.6075)
(250000, 30.9827)
(300000, 32.0676)
(350000, 30.2722)
(400000, 29.6016)
}
\defdata{e_xs_2-guid---steps-sfid10k-argmin}{8}
\defdata{e_xs_2-guid---steps-sfid10k-min}{29.6016}

\defdata{e_xs_2-guid---steps-precision}{
(50000,  0.3146)
(100000, 0.4032)
(150000, 0.4335)
(200000, 0.4565)
(250000, 0.4785)
(300000, 0.5055)
(350000, 0.5298)
(400000, 0.5207)
}
\defdata{e_xs_2-guid---steps-precision-argmax}{7}
\defdata{e_xs_2-guid---steps-precision-max}{0.5298}

\defdata{e_xs_2-guid---steps-recall}{
(50000,  0.2173)
(100000, 0.2676)
(150000, 0.2933)
(200000, 0.2992)
(250000, 0.2934)
(300000, 0.2898)
(350000, 0.2826)
(400000, 0.3077)
}
\defdata{e_xs_2-guid---steps-recall-argmax}{8}
\defdata{e_xs_2-guid---steps-recall-max}{0.3077}
\defdata{a_s_4-noguid---steps-is}{
(50000,  7.6988)
(100000, 8.2830)
(150000, 8.6896)
(200000, 9.4085)
(250000, 9.5766)
(300000, 9.4797)
(350000, 9.6073)
(400000, 9.8452)
}
\defdata{a_s_4-noguid---steps-is-argmax}{8}
\defdata{a_s_4-noguid---steps-is-max}{}

\defdata{a_s_4-noguid---steps-fid10k}{
(50000,  122.4451)
(100000, 116.7441)
(150000, 111.1474)
(200000, 108.2022)
(250000, 104.8131)
(300000, 104.6816)
(350000, 105.7584)
(400000, 103.7711)
}
\defdata{a_s_4-noguid---steps-fid10k-argmin}{8}
\defdata{a_s_4-noguid---steps-fid10k-min}{}

\defdata{a_s_4-noguid---steps-sfid10k}{
(50000,  44.9435)
(100000, 44.5687)
(150000, 40.6642)
(200000, 37.7334)
(250000, 36.0744)
(300000, 34.9803)
(350000, 36.5966)
(400000, 35.4860)
}
\defdata{a_s_4-noguid---steps-sfid10k-argmin}{6}
\defdata{a_s_4-noguid---steps-sfid10k-min}{}

\defdata{a_s_4-noguid---steps-precision}{
(50000,  0.1305)
(100000, 0.1499)
(150000, 0.1466)
(200000, 0.1578)
(250000, 0.1752)
(300000, 0.1698)
(350000, 0.1748)
(400000, 0.1712)
}
\defdata{a_s_4-noguid---steps-precision-argmax}{5}
\defdata{a_s_4-noguid---steps-precision-max}{}

\defdata{a_s_4-noguid---steps-recall}{
(50000,  0.1405)
(100000, 0.1855)
(150000, 0.1736)
(200000, 0.2000)
(250000, 0.2061)
(300000, 0.2158)
(350000, 0.2312)
(400000, 0.2287)
}
\defdata{a_s_4-noguid---steps-recall-argmax}{7}
\defdata{a_s_4-noguid---steps-recall-max}{}

\defdata{a_s_4-guid---steps-is}{
(50000,  15.9805)
(100000, 21.3098)
(150000, 24.5366)
(200000, 27.8785)
(250000, 31.4257)
(300000, 34.1889)
(350000, 36.6818)
(400000, 37.6940)
}
\defdata{a_s_4-guid---steps-is-argmax}{8}
\defdata{a_s_4-guid---steps-is-max}{}

\defdata{a_s_4-guid---steps-fid10k}{
(50000,  72.5042)
(100000, 60.9087)
(150000, 55.6439)
(200000, 52.0162)
(250000, 48.0001)
(300000, 45.5693)
(350000, 43.8889)
(400000, 42.8074)
}
\defdata{a_s_4-guid---steps-fid10k-argmin}{8}
\defdata{a_s_4-guid---steps-fid10k-min}{}

\defdata{a_s_4-guid---steps-sfid10k}{
(50000,  44.7188)
(100000, 37.0649)
(150000, 36.4826)
(200000, 35.5373)
(250000, 34.0756)
(300000, 33.7100)
(350000, 33.3794)
(400000, 31.8999)
}
\defdata{a_s_4-guid---steps-sfid10k-argmin}{8}
\defdata{a_s_4-guid---steps-sfid10k-min}{}

\defdata{a_s_4-guid---steps-precision}{
(50000,  0.2084)
(100000, 0.2777)
(150000, 0.3172)
(200000, 0.3424)
(250000, 0.3765)
(300000, 0.3882)
(350000, 0.4073)
(400000, 0.4084)
}
\defdata{a_s_4-guid---steps-precision-argmax}{8}
\defdata{a_s_4-guid---steps-precision-max}{}

\defdata{a_s_4-guid---steps-recall}{
(50000,  0.1809)
(100000, 0.1886)
(150000, 0.2170)
(200000, 0.2395)
(250000, 0.2429)
(300000, 0.2512)
(350000, 0.2624)
(400000, 0.2806)
}
\defdata{a_s_4-guid---steps-recall-argmax}{8}
\defdata{a_s_4-guid---steps-recall-max}{}

\defdata{b_s_4-noguid---steps-is}{
(50000,  7.6984)
(100000, 8.2711)
(150000, 8.7226)
(200000, 9.5126)
(250000, 9.4777)
(300000, 9.8898)
(350000, 9.7527)
(400000, 10.0052)
}
\defdata{b_s_4-noguid---steps-is-argmax}{8}
\defdata{b_s_4-noguid---steps-is-max}{}

\defdata{b_s_4-noguid---steps-fid10k}{
(50000,  121.6621)
(100000, 116.1561)
(150000, 111.2853)
(200000, 106.7796)
(250000, 103.3811)
(300000, 102.0635)
(350000, 104.8763)
(400000, 102.8196)
}
\defdata{b_s_4-noguid---steps-fid10k-argmin}{6}
\defdata{b_s_4-noguid---steps-fid10k-min}{}

\defdata{b_s_4-noguid---steps-sfid10k}{
(50000,  45.5931)
(100000, 43.2324)
(150000, 39.8737)
(200000, 36.5901)
(250000, 35.5096)
(300000, 33.8736)
(350000, 36.3729)
(400000, 34.9780)
}
\defdata{b_s_4-noguid---steps-sfid10k-argmin}{6}
\defdata{b_s_4-noguid---steps-sfid10k-min}{}

\defdata{b_s_4-noguid---steps-precision}{
(50000,  0.1369)
(100000, 0.1504)
(150000, 0.1444)
(200000, 0.1593)
(250000, 0.1776)
(300000, 0.1702)
(350000, 0.1690)
(400000, 0.1738)
}
\defdata{b_s_4-noguid---steps-precision-argmax}{5}
\defdata{b_s_4-noguid---steps-precision-max}{}

\defdata{b_s_4-noguid---steps-recall}{
(50000,  0.2089)
(100000, 0.1766)
(150000, 0.1926)
(200000, 0.1915)
(250000, 0.1867)
(300000, 0.2325)
(350000, 0.2439)
(400000, 0.2313)
}
\defdata{b_s_4-noguid---steps-recall-argmax}{7}
\defdata{b_s_4-noguid---steps-recall-max}{}

\defdata{b_s_4-guid---steps-is}{
(50000,  15.9544)
(100000, 21.5633)
(150000, 25.0941)
(200000, 29.0342)
(250000, 32.3721)
(300000, 35.3517)
(350000, 37.6852)
(400000, 39.0420)
}
\defdata{b_s_4-guid---steps-is-argmax}{8}
\defdata{b_s_4-guid---steps-is-max}{}

\defdata{b_s_4-guid---steps-fid10k}{
(50000,  70.6872)
(100000, 59.6518)
(150000, 54.0261)
(200000, 51.0015)
(250000, 47.4610)
(300000, 44.8883)
(350000, 42.9043)
(400000, 42.0182)
}
\defdata{b_s_4-guid---steps-fid10k-argmin}{8}
\defdata{b_s_4-guid---steps-fid10k-min}{}

\defdata{b_s_4-guid---steps-sfid10k}{
(50000,  43.9647)
(100000, 36.4637)
(150000, 35.7989)
(200000, 35.6300)
(250000, 33.5075)
(300000, 33.9849)
(350000, 33.4847)
(400000, 32.4949)
}
\defdata{b_s_4-guid---steps-sfid10k-argmin}{8}
\defdata{b_s_4-guid---steps-sfid10k-min}{}

\defdata{b_s_4-guid---steps-precision}{
(50000,  0.2273)
(100000, 0.2869)
(150000, 0.3212)
(200000, 0.3535)
(250000, 0.3842)
(300000, 0.4011)
(350000, 0.4145)
(400000, 0.4207)
}
\defdata{b_s_4-guid---steps-precision-argmax}{8}
\defdata{b_s_4-guid---steps-precision-max}{}

\defdata{b_s_4-guid---steps-recall}{
(50000,  0.1643)
(100000, 0.1873)
(150000, 0.2177)
(200000, 0.2278)
(250000, 0.2359)
(300000, 0.2524)
(350000, 0.2747)
(400000, 0.2830)
}
\defdata{b_s_4-guid---steps-recall-argmax}{8}
\defdata{b_s_4-guid---steps-recall-max}{}

\defdata{c_s_4-noguid---steps-is}{
(50000,  8.7913)
(100000, 9.5804)
(150000, 9.8638)
(200000, 10.7746)
(250000, 10.6813)
(300000, 10.6831)
(350000, 11.1917)
(400000, 11.5418)
}
\defdata{c_s_4-noguid---steps-is-argmax}{8}
\defdata{c_s_4-noguid---steps-is-max}{}

\defdata{c_s_4-noguid---steps-fid10k}{
(50000,  112.1645)
(100000, 102.8959)
(150000, 101.4965)
(200000, 96.1803)
(250000, 95.2832)
(300000, 96.5717)
(350000, 94.1852)
(400000, 90.6072)
}
\defdata{c_s_4-noguid---steps-fid10k-argmin}{8}
\defdata{c_s_4-noguid---steps-fid10k-min}{}

\defdata{c_s_4-noguid---steps-sfid10k}{
(50000,  38.1646)
(100000, 36.1729)
(150000, 34.7426)
(200000, 32.5354)
(250000, 35.2231)
(300000, 33.8559)
(350000, 32.6480)
(400000, 29.8746)
}
\defdata{c_s_4-noguid---steps-sfid10k-argmin}{8}
\defdata{c_s_4-noguid---steps-sfid10k-min}{}

\defdata{c_s_4-noguid---steps-precision}{
(50000,  0.1579)
(100000, 0.1770)
(150000, 0.1876)
(200000, 0.1794)
(250000, 0.1871)
(300000, 0.1830)
(350000, 0.1944)
(400000, 0.1840)
}
\defdata{c_s_4-noguid---steps-precision-argmax}{7}
\defdata{c_s_4-noguid---steps-precision-max}{}

\defdata{c_s_4-noguid---steps-recall}{
(50000,  0.1813)
(100000, 0.2312)
(150000, 0.2340)
(200000, 0.2643)
(250000, 0.2673)
(300000, 0.2661)
(350000, 0.2845)
(400000, 0.2911)
}
\defdata{c_s_4-noguid---steps-recall-argmax}{8}
\defdata{c_s_4-noguid---steps-recall-max}{}

\defdata{c_s_4-guid---steps-is}{
(50000,  19.1649)
(100000, 25.1845)
(150000, 29.7781)
(200000, 33.9933)
(250000, 37.6071)
(300000, 39.0798)
(350000, 43.1623)
(400000, 45.6331)
}
\defdata{c_s_4-guid---steps-is-argmax}{8}
\defdata{c_s_4-guid---steps-is-max}{}

\defdata{c_s_4-guid---steps-fid10k}{
(50000,  63.4484)
(100000, 53.3350)
(150000, 47.4637)
(200000, 44.5850)
(250000, 41.3353)
(300000, 39.8213)
(350000, 38.3976)
(400000, 37.2138)
}
\defdata{c_s_4-guid---steps-fid10k-argmin}{8}
\defdata{c_s_4-guid---steps-fid10k-min}{}

\defdata{c_s_4-guid---steps-sfid10k}{
(50000,  36.7959)
(100000, 33.1362)
(150000, 32.0042)
(200000, 32.7330)
(250000, 30.5309)
(300000, 32.4509)
(350000, 31.2436)
(400000, 31.4548)
}
\defdata{c_s_4-guid---steps-sfid10k-argmin}{5}
\defdata{c_s_4-guid---steps-sfid10k-min}{}

\defdata{c_s_4-guid---steps-precision}{
(50000,  0.2669)
(100000, 0.3283)
(150000, 0.3788)
(200000, 0.3971)
(250000, 0.4405)
(300000, 0.4511)
(350000, 0.4689)
(400000, 0.4848)
}
\defdata{c_s_4-guid---steps-precision-argmax}{8}
\defdata{c_s_4-guid---steps-precision-max}{}

\defdata{c_s_4-guid---steps-recall}{
(50000,  0.1872)
(100000, 0.2265)
(150000, 0.2512)
(200000, 0.2579)
(250000, 0.2754)
(300000, 0.2918)
(350000, 0.2864)
(400000, 0.2929)
}
\defdata{c_s_4-guid---steps-recall-argmax}{8}
\defdata{c_s_4-guid---steps-recall-max}{}

\defdata{d_s_4-noguid---steps-is}{
(50000,  8.4709)
(100000, 9.4570)
(150000, 10.1958)
(200000, 10.4904)
(250000, 10.9742)
(300000, 10.8274)
(350000, 11.1706)
(400000, 10.9714)
}
\defdata{d_s_4-noguid---steps-is-argmax}{7}
\defdata{d_s_4-noguid---steps-is-max}{}

\defdata{d_s_4-noguid---steps-fid10k}{
(50000,  113.8789)
(100000, 103.3315)
(150000, 100.3675)
(200000, 96.9317)
(250000, 90.6293)
(300000, 93.1310)
(350000, 92.2255)
(400000, 94.1036)
}
\defdata{d_s_4-noguid---steps-fid10k-argmin}{5}
\defdata{d_s_4-noguid---steps-fid10k-min}{}

\defdata{d_s_4-noguid---steps-sfid10k}{
(50000,  40.0189)
(100000, 36.7225)
(150000, 32.9107)
(200000, 34.1700)
(250000, 31.6265)
(300000, 31.8346)
(350000, 32.8247)
(400000, 31.6264)
}
\defdata{d_s_4-noguid---steps-sfid10k-argmin}{8}
\defdata{d_s_4-noguid---steps-sfid10k-min}{}

\defdata{d_s_4-noguid---steps-precision}{
(50000,  0.1677)
(100000, 0.1746)
(150000, 0.1803)
(200000, 0.1763)
(250000, 0.1921)
(300000, 0.1786)
(350000, 0.2000)
(400000, 0.1816)
}
\defdata{d_s_4-noguid---steps-precision-argmax}{7}
\defdata{d_s_4-noguid---steps-precision-max}{}

\defdata{d_s_4-noguid---steps-recall}{
(50000,  0.1632)
(100000, 0.2212)
(150000, 0.2331)
(200000, 0.2901)
(250000, 0.2859)
(300000, 0.2749)
(350000, 0.3145)
(400000, 0.2856)
}
\defdata{d_s_4-noguid---steps-recall-argmax}{7}
\defdata{d_s_4-noguid---steps-recall-max}{}

\defdata{d_s_4-guid---steps-is}{
(50000,  18.7588)
(100000, 25.2494)
(150000, 29.8707)
(200000, 33.4782)
(250000, 37.7808)
(300000, 40.5606)
(350000, 42.5407)
(400000, 43.6548)
}
\defdata{d_s_4-guid---steps-is-argmax}{8}
\defdata{d_s_4-guid---steps-is-max}{}

\defdata{d_s_4-guid---steps-fid10k}{
(50000,  63.3218)
(100000, 54.2534)
(150000, 48.3369)
(200000, 44.6362)
(250000, 40.2810)
(300000, 39.4045)
(350000, 37.9931)
(400000, 37.9712)
}
\defdata{d_s_4-guid---steps-fid10k-argmin}{8}
\defdata{d_s_4-guid---steps-fid10k-min}{}

\defdata{d_s_4-guid---steps-sfid10k}{
(50000,  36.7239)
(100000, 33.1069)
(150000, 32.0208)
(200000, 31.8918)
(250000, 30.1826)
(300000, 32.0520)
(350000, 31.2785)
(400000, 31.7648)
}
\defdata{d_s_4-guid---steps-sfid10k-argmin}{5}
\defdata{d_s_4-guid---steps-sfid10k-min}{}

\defdata{d_s_4-guid---steps-precision}{
(50000,  0.2604)
(100000, 0.3323)
(150000, 0.3760)
(200000, 0.3969)
(250000, 0.4585)
(300000, 0.4617)
(350000, 0.4657)
(400000, 0.4770)
}
\defdata{d_s_4-guid---steps-precision-argmax}{8}
\defdata{d_s_4-guid---steps-precision-max}{}

\defdata{d_s_4-guid---steps-recall}{
(50000,  0.1877)
(100000, 0.2289)
(150000, 0.2630)
(200000, 0.2797)
(250000, 0.3070)
(300000, 0.2910)
(350000, 0.2933)
(400000, 0.3004)
}
\defdata{d_s_4-guid---steps-recall-argmax}{5}
\defdata{d_s_4-guid---steps-recall-max}{}

\defdata{e_s_4-noguid---steps-is}{
(50000,  8.0433)
(100000, 9.2928)
(150000, 9.6581)
(200000, 10.4419)
(250000, 10.4978)
(300000, 10.2694)
(350000, 11.2459)
(400000, 10.7104)
}
\defdata{e_s_4-noguid---steps-is-argmax}{7}
\defdata{e_s_4-noguid---steps-is-max}{}

\defdata{e_s_4-noguid---steps-fid10k}{
(50000,  119.8354)
(100000, 104.9984)
(150000, 103.0871)
(200000, 98.8602)
(250000, 97.6167)
(300000, 98.1541)
(350000, 91.6412)
(400000, 98.2376)
}
\defdata{e_s_4-noguid---steps-fid10k-argmin}{7}
\defdata{e_s_4-noguid---steps-fid10k-min}{}

\defdata{e_s_4-noguid---steps-sfid10k}{
(50000,  42.0304)
(100000, 36.0620)
(150000, 34.2586)
(200000, 34.8408)
(250000, 33.8734)
(300000, 32.6469)
(350000, 31.7836)
(400000, 32.4118)
}
\defdata{e_s_4-noguid---steps-sfid10k-argmin}{7}
\defdata{e_s_4-noguid---steps-sfid10k-min}{}

\defdata{e_s_4-noguid---steps-precision}{
(50000,  0.1670)
(100000, 0.1696)
(150000, 0.1759)
(200000, 0.1610)
(250000, 0.1765)
(300000, 0.1708)
(350000, 0.1844)
(400000, 0.1674)
}
\defdata{e_s_4-noguid---steps-precision-argmax}{7}
\defdata{e_s_4-noguid---steps-precision-max}{}

\defdata{e_s_4-noguid---steps-recall}{
(50000,  0.1604)
(100000, 0.2341)
(150000, 0.2195)
(200000, 0.2414)
(250000, 0.2741)
(300000, 0.2730)
(350000, 0.3024)
(400000, 0.2655)
}
\defdata{e_s_4-noguid---steps-recall-argmax}{7}
\defdata{e_s_4-noguid---steps-recall-max}{}

\defdata{e_s_4-guid---steps-is}{
(50000,  17.1709)
(100000, 22.1483)
(150000, 26.0985)
(200000, 29.8000)
(250000, 31.8721)
(300000, 35.4329)
(350000, 37.9104)
(400000, 39.1723)
}
\defdata{e_s_4-guid---steps-is-argmax}{8}
\defdata{e_s_4-guid---steps-is-max}{}

\defdata{e_s_4-guid---steps-fid10k}{
(50000,  67.5076)
(100000, 57.7965)
(150000, 51.9083)
(200000, 47.3574)
(250000, 44.6656)
(300000, 41.8128)
(350000, 41.2655)
(400000, 40.2836)
}
\defdata{e_s_4-guid---steps-fid10k-argmin}{8}
\defdata{e_s_4-guid---steps-fid10k-min}{}

\defdata{e_s_4-guid---steps-sfid10k}{
(50000,  37.4570)
(100000, 34.1979)
(150000, 34.3507)
(200000, 33.5809)
(250000, 32.7193)
(300000, 31.6471)
(350000, 32.4275)
(400000, 31.9747)
}
\defdata{e_s_4-guid---steps-sfid10k-argmin}{6}
\defdata{e_s_4-guid---steps-sfid10k-min}{}

\defdata{e_s_4-guid---steps-precision}{
(50000,  0.2579)
(100000, 0.3024)
(150000, 0.3346)
(200000, 0.3664)
(250000, 0.4021)
(300000, 0.4265)
(350000, 0.4437)
(400000, 0.4640)
}
\defdata{e_s_4-guid---steps-precision-argmax}{8}
\defdata{e_s_4-guid---steps-precision-max}{}

\defdata{e_s_4-guid---steps-recall}{
(50000,  0.1775)
(100000, 0.2208)
(150000, 0.2540)
(200000, 0.2655)
(250000, 0.2770)
(300000, 0.2644)
(350000, 0.2838)
(400000, 0.2870)
}
\defdata{e_s_4-guid---steps-recall-argmax}{8}
\defdata{e_s_4-guid---steps-recall-max}{}

\title[Exploring Magnitude Preservation and Rotation Modulation]{Exploring Magnitude Preservation and Rotation Modulation in Diffusion Transformers}

\hldauthor{
  \begin{NoHyper}
    \hspace{-0.75em}
    \Name{Eric Tillman Bill}\thanks{Equal contribution. Author ordering determined by coin flip.} 
    \Email{erbill@ethz.ch} \\
    \addr{ETH Zürich, Zürich, Switzerland}
  \end{NoHyper}
  \AND
  \Name{Cristian {Perez Jensen}}\footnotemark[1] \Email{cjense@ethz.ch} \\
  \addr{ETH Zürich, Zürich, Switzerland}
  \AND
  \Name{Sotiris Anagnostidis} \Email{sotirios.anagnostidis@inf.ethz.ch} \\
  \addr{ETH Zürich, Zürich, Switzerland}
  \AND
  \Name{Dimitri {von Rütte}} \Email{dimitri.vonrutte@inf.ethz.ch} \\
  \addr{ETH Zürich, Zürich, Switzerland}
}

\begin{document}

\maketitle
\vspace{-1.5cm}
\begin{abstract}%
    Denoising diffusion models exhibit remarkable generative capabilities, but remain challenging to train due to their inherent stochasticity, where high-variance gradient estimates lead to slow convergence. Previous works have shown that magnitude preservation helps with stabilizing training in the U-net architecture. This work explores whether this effect extends to the Diffusion Transformer (DiT) architecture. As such, we propose a magnitude-preserving design that stabilizes training without normalization layers. Motivated by the goal of maintaining activation magnitudes, we additionally introduce rotation modulation, which is a novel conditioning method using learned rotations instead of traditional scaling or shifting. Through empirical evaluations and ablation studies on small-scale models, we show that magnitude-preserving strategies significantly improve performance, notably reducing FID scores by $\sim$12.8\%. Further, we show that rotation modulation combined with scaling is competitive with AdaLN, while requiring $\sim$5.4\% fewer parameters. This work provides insights into conditioning strategies and magnitude control. We will publicly release the implementation of our method.
\end{abstract}

\begin{keywords}%
    Diffusion Transformer, Magnitude Preservation, Condition Modulation.
\end{keywords}

\vspace{-0.2cm}

\section{Introduction}\label{sec:intro}

\begin{wrapfigure}{r}{0.44\linewidth}
    \centering
    \vspace{-7mm}
    \begin{subfigure}[t]{0.4\linewidth}
        \centering
        \captionsetup{labelformat=empty}
        \caption{\textbf{DiT}}
        \includegraphics[width=\linewidth]{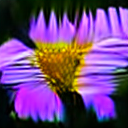}
    \end{subfigure}
    \hspace{10pt}
    \begin{subfigure}[t]{0.4\linewidth}
        \centering
        \captionsetup{labelformat=empty}
        \caption{\textbf{MaP-DiT}}
        \includegraphics[width=\linewidth]{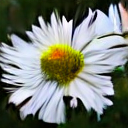}
    \end{subfigure}
    \caption{Effect of magnitude preservation and weight control on DiT-S/4.}
    \label{tab:intro_example}
    \vspace{-3mm}
\end{wrapfigure}

Denoising diffusion models \cite{ho2020denoising,sohl2015deep,song2020score} have gained prominence for their remarkable generative capabilities but remain challenging to train due to their iterative and stochastic nature, where high-variance gradient estimates lead to slow convergence. At the core of this process is the task of predicting a Gaussian noise vector with unit magnitude at each timestep. This target implicitly defines a magnitude constraint on the model's outputs. Yet, most diffusion architectures do not enforce this constraint at the architectural level.

We hypothesize that aligning the model's internal representations with this denoising objective can lead to more stable training. In particular, preserving activation magnitudes across the network may help maintain signal consistency and reduce training instabilities caused by uncontrolled norm growth or collapse.

Recent work by \citet{karras_analyzing_2024} supports this view, showing that controlling activation magnitudes layer-by-layer in the ADM architecture \citep{dhariwal_diffusion_2021} significantly improves performance. They advocate for extending these techniques to other architectures, including transformer-based diffusion models, such as \acp{dit}. This architecture replaces the conventional U-Net backbone with a transformer-based design, building on the \ac{vit}.

While \citet{karras_analyzing_2024} focused on achieving unit activation magnitude in each ADM layer, their methods cannot be directly applied. The \ac{dit} architecture incorporates \ac{adaln} modulation blocks, which scale and shift inputs based on label and timestep conditioning. These modulation blocks are the only way that the conditioning variables influence the output. These adaptions call for extending the activation control techniques to accommodate arbitrary magnitudes.

To this end, we propose architecture modifications that ensures activation magnitudes are preserved across all components of the model, assuming normalized training data as input, while avoiding explicit activation normalization in every intermediate step. Additionally, we constrain model weight growth by enforcing unit magnitude constraints on parameters. Our approach ensures stable training dynamics without introducing activation biases or requiring manual adjustments. In summary, our contributions are as follows:
\begin{itemize}[topsep=1ex,itemsep=0ex,partopsep=1ex,parsep=1ex]
    \item We extend the magnitude preserving techniques by extending their proofs to support arbitrary magnitudes without any assumption on the underlying magnitude, enabling their application to transformer-based architectures like \ac{dit}.
    
    \item Furthermore, we conduct comprehensive ablation evaluations, showing that magnitude preservation and weight control techniques demonstrate significant performance gains across various metrics, including notable improvements in FID-10K scores.
    
    \item Lastly, we introduce a novel modulation block: \emph{rotation modulation}. Instead of learning a scaling or translation, it learns a rotation, based on the conditioning variables. We conduct an ablation study of the three types of modulation.
\end{itemize}
\section{Magnitude Preservation and Weight Growth Control}\label{sec:mp}
Throughout our work, we adopt the concept of \textit{expected magnitude}, introduced by \citet{karras_analyzing_2024}, which is defined for a random multivariate variable $\vec{x} \in \R^n$ as
\begin{equation}
    \MP{\vec{x}} \defeq \sqrt{\frac{1}{n} \sum_{i=1}^{n} \E\lft[x_i^2\rgt]}.
\end{equation}

The overarching goal is to ensure that every component of our modified DiT architecture operates in a magnitude-preserving manner. Specifically, for any component $f$, we design a magnitude preserving version $f_\text{MP}$, such that for an any input $\vec{x}$, we have $\MP{f_\text{MP}(\vec{x})} = \MP{\vec{x}}$. Except for the attention mechanism, we show that it is possible to generalize the magnitude-preservation techniques from \citet{karras_analyzing_2024} to accommodate arbitrary magnitudes $\MP{\vec{x}} = \sigma$, instead of restricting them to unit magnitude $\MP{\vec{x}} = 1$. Derivations for all techniques are provided in \Cref{sec:proofs}.

\begin{lemma} \label{lem:att}
    Let $\mat{A} \in \R^{T \times T}$ be an unnormalized attention map and $\mat{V} \in \R^{T \times n}$ with $\MP{\vec{v}_t} = \sigma$ for all $t \in [T]$. Further, define attention as
    \begin{equation}
        \operatorname{att}(\mat{A}, \mat{V}) \defeq \operatorname{softmax}_{\beta}  (\mat{A}) \mat{V}, \quad \operatorname{softmax}_{\beta}(\mat{A})_{ij} \defeq \frac{\exp(\nicefrac{a_{ij}}{\beta})}{\sum_{k=1}^n \exp(\nicefrac{a_{ik}}{\beta})}.
    \end{equation}
    Then, $\MP{\operatorname{att}(\mat{A},\mat{V})_t} \leq \sigma$ for all $t \in [T]$. As $\beta \to 0$, the inequality becomes equality.
\end{lemma}

\paragraph{Self-attention}
In self-attention layers, \citet{karras_analyzing_2024} propose to use \emph{cosine attention} \citep{luo_cosine_2017, nguyen_enhancing_2023}, along with activation normalization. While cosine attention aligns with our goals of controlling activation growth, normalizing activations does not. As such, we replace the dot-product attention by cosine attention, where similarity between vectors $\vec{x}$ and $\vec{y}$ is computed by
\begin{equation}
    \mathrm{cosim}(\vec{x}, \vec{y}) \defeq \cos(\phi_{\vec{x},\vec{y}}) = \frac{\vec{x}^\top \vec{y}}{\| \vec{x} \| \| \vec{y} \|} \in [-1, 1].
\end{equation}
Here, $\phi_{\vec{x}, \vec{y}}$ denotes the angle between $\vec{x}$ and $\vec{y}$. This notion of similarity solely depends on the direction of the vectors and not on their length, preventing uncontrollable growth of activations. \Cref{lem:att} shows that the attention mechanism only decreases the input magnitude under any attention mechanism, such as cosine and dot-product attention. Hence, we argue that normalizing activations are not necessary, since the attention mechanism cannot increase the magnitude of its input anyways.

In conclusion, under the assumptions of Lemmas \ref{lem:att} to \ref{lem:silu}, we are able to show for all layers that the output magnitude is equal to or upper bounded by the input magnitude. Therefore, at least at initialization when all assumptions hold true, magnitude cannot increase throughout the model.

\paragraph{Weight growth control}
\citet{karras_analyzing_2024} showed that controlling weight growth is crucial for ensuring stable training dynamics. Because of this, we believe the same holds true for the DiT architecture. They propose using forced weight normalization, such that after each training step, each weight vector is normalized to unit magnitude. Specifically, for each linear projection with a weight matrix $\mat{W} \in \R^{n \times m}$, the normalization is applied as $\vec{w}_i \leftarrow \nicefrac{\vec{w}_i}{\|\vec{w}_i\|}$. This operation ensures that each weight vector is normalized to lie on the unit-magnitude hypersphere. By doing so, the training gradients are effectively projected onto the tangent plane of this hypersphere, preventing the weights from growing unbounded while preserving their directionality.

\section{Modulation} \label{sec:modulation}
Following \citet{peebles_scalable_2023}, we condition the latent image tokens $\mathbf{X}$ on their class labels via \ac{adaln}. In each \ac{dit} block, \ac{adaln} learns three vectors (scale, shift, and gate $\vec{s}, \vec{b}, \vec{g} \in \R^n$) for both the attention and MLP, and applies them as follows
\begin{equation}
    \vec{x}_t \gets \vec{g} \odot \mathrm{layer}\lft(\vec{s} \odot \vec{x}_t + \vec{b} \rgt),
\end{equation}
for each $t \in [T]$. While \ac{adaln} is effective, it can disrupt the magnitude of the latent tokens, because the shift may move activations away from zero and the scale can amplify magnitude. Although our theoretical results remain valid under arbitrary magnitudes, a more natural and principled conditioning method that inherently preserves magnitude is desirable. To this end, we propose \emph{rotation modulation}. Instead of scaling and shifting, it predicts a set of rotation angles and applies them to the latent tokens.

\begin{lemma} \label{lem:rot_mod}
    Let $\mat{R} \in \R^{n \times n}$ be a rotation matrix, such that $\trp{\mat{R}} = \mat{R}^{-1}$ and $\det{\mat{R}} = 1$. Further, let $\vec{x} \in \R^n$ be the input vector, then
    \begin{equation}
        \MP{\mat{R} \vec{x}} = \MP{\vec{x}}.
    \end{equation}
\end{lemma}
See \Cref{sec:proofs} for a proof. Since a full $n$-dimensional rotation requires $\nicefrac{n(n-1)}{2}$ degrees of freedom, we partition each token into $\nicefrac{d}{2}$ disjoint two-dimensional sub-vectors and predict a single rotation angle per pair. Specifically, for each sub-vector $\vec{x}_{t,[i,i+1]} \in \R^2$, the model learns to predict an angle $\theta_i \in \R$ and modulates the sub-vector similarly to Rotary Position Embeddings (RoPE) \citep{su_roformer_2023}:
\begin{equation}
    \vec{x}_{t,[i,i+1]} \gets \begin{bmatrix}
        \cos \theta_i & -\sin \theta_i \\
        \sin \theta_i & \cos \theta_i
    \end{bmatrix} \vec{x}_{t,[i,i+1]}.
\end{equation}
As a result, this ``patchified'' rotation requires only $\nicefrac{d}{2}$ parameters. In contrast to previous variants, rotation modulation provides a method for conditioning on auxiliary variables, while preserving magnitude (see \Cref{lem:rot_mod}).
\NewDocumentCommand\Four{smm}{%
    \begin{subfigure}[b]{0.19\textwidth}
        \centering
        \captionsetup{labelformat=empty}
        \IfBooleanTF{#1}{%
            \caption{#3}
            \includegraphics[width=\textwidth]{figures/four/#2}%
        }{%
            \includegraphics[width=\textwidth]{figures/four/#2}
            \caption{#3}%
        }
        \label{fig:four-#2}
    \end{subfigure}%
}

\section{Results}\label{sec:res}

\begin{figure*}[t]
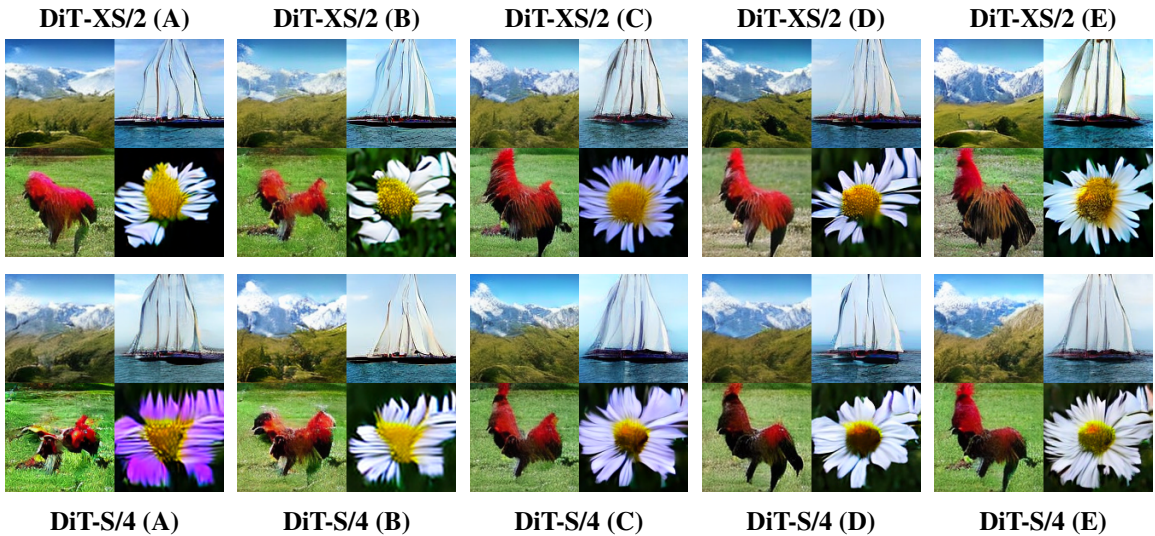

    \centering
    \Four*{A-XS-2}{\textbf{\DIT{XS}{2} (A)}} \hfill
    \Four*{B-XS-2}{\textbf{DiT-XS/2 (B)}} \hfill
    \Four*{C-XS-2}{\textbf{DiT-XS/2 (C)}} \hfill
    \Four*{D-XS-2}{\textbf{DiT-XS/2 (D)}} \hfill
    \Four*{E-XS-2}{\textbf{DiT-XS/2 (E)}} \\[-8pt]
    \Four{A-S-4}{\textbf{DiT-S/4 (A)}} \hfill
    \Four{B-S-4}{\textbf{DiT-S/4 (B)}} \hfill
    \Four{C-S-4}{\textbf{DiT-S/4 (C)}} \hfill
    \Four{D-S-4}{\textbf{DiT-S/4 (D)}} \hfill
    \Four{E-S-4}{\textbf{DiT-S/4 (E)}} \\
    \caption{Four samples from all models and configurations, where all samples were generated with the same seed, a guidance scale of 5.0, and an EMA relative standard deviation of \SI{10}{\percent}. The samples include ``alp'' (970), ``schooner'' (780), ``cock'' (7), and ``daisy'' (985).}
    \label{fig:main-samples}
\end{figure*}

In order to accommodate limited compute, we used the $128 \times 128$ variant of ImageNet
\citep{ILSVRC15}. To further accelerate training, we preprocessed the dataset by encoding each
image with a pre-trained \ac{vae}. Further, due to limited resources, we experimented only with
\DIT{XS}{2}, \DIT{S}{4}, and \DIT{S}{2} models and trained them for 400K steps. The former two were
used to ablate the magnitude preservation techniques, while the latter was used for ablations of
the modulation techniques.

For evaluation, we mainly consider \ac{fid}, computed over 10K samples between the ground truth
dataset and model samples. In \Cref{sec:convergence}, we also report \ac{sfid}, precision, recall,
and inception scores. Note that while the use of smaller models and lower-resolution datasets
limits compatibility with larger models in prior work, the focus is on analyzing architectural
effects under similarly constrained model sizes.

\paragraph{Magnitude preservation}
\begin{wraptable}{r}{0.55\linewidth}
    \sisetup{detect-weight=true}
    \robustify\bfseries
    \centering
    \vspace{-5mm}
    \caption{FID-10K scores for all configurations.}
    \begin{tabular}{@{}l@{\hspace{0.5em}}lS[table-format=2.2]S[table-format=3.2]@{}}
        \toprule
        \multicolumn{2}{@{}l}{\textbf{Config}} & \textbf{DiT-XS-2}        & \textbf{DiT-S-4}                   \\
        \midrule
        A                                      & Baseline                 & 99.43            & 103.77          \\
        B                                      & + Cosine attention       & 99.14            & 102.82          \\
        C                                      & + Magnitude preservation & \bfseries 86.49  & \bfseries 90.61 \\
        D                                      & + Weight growth control  & 89.38            & 90.63           \\
        E                                      & + No normalization       & 86.72            & 91.64           \\
        \bottomrule
    \end{tabular}
    \label{tab:main-result}
\end{wraptable}
We define five configurations (A–E), with Config~A as the \ac{dit} baseline. Each subsequent configuration adds one of our techniques, culminating in Config~E, which incorporates all modifications and removes activation normalization. For full details, consult \Cref{tab:dit_attributes} in \Cref{sec:configurations}. Our results show that magnitude-preserving techniques notably improve performance. As shown in \Cref{fig:main-samples}, configurations C--E produce higher quality images compared to configurations A and B. This qualitative improvement is supported by significantly lower \ac{fid} scores, as shown in \Cref{tab:main-result}.  This trend holds across additional metrics such as Inception score, precision, recall, and \acs{sfid}-10K.  For a more detailed analysis, refer to \Cref{sec:convergence}. However, these improvements come with increased training cost. Config~E incurs an average runtime overhead of $\SI{8.5}{\percent}$ compared to the original \ac{dit} implementation. Additionally, we visualize the evolution of activation magnitudes before and after training in \Cref{sec:mag_evol}, offering insight into the underlying causes of performance differences.

\paragraph{Modulation}
\begin{wraptable}{r}{0.38\linewidth}
    \sisetup{detect-weight=true}
    \robustify\bfseries
    \centering
    \vspace{-5mm}
    \caption{Performance of various modulation combinations on \DIT{XS}{2}.}
    \label{tab:mod}
    \begin{tabular}{@{}cccS[table-format=2.2]@{}}
        \toprule
        \textbf{Scale} & \textbf{Shift} & \textbf{Rotate} & \textbf{\Acs{fid} $\downarrow$} \\
        \midrule
        \cmark         & \xmark         & \xmark          & 72.03                           \\
        \xmark         & \cmark         & \xmark          & 85.23                           \\
        \xmark         & \xmark         & \cmark          & 84.62                           \\
        \cmark         & \cmark         & \xmark          & \bfseries 69.28                 \\
        \cmark         & \xmark         & \cmark          & 70.86                           \\
        \xmark         & \cmark         & \cmark          & 74.01                           \\
        \cmark         & \cmark         & \cmark          & 72.19                           \\
        \bottomrule
    \end{tabular}
    \vspace{-4mm}
\end{wraptable}
To isolate the effect of different modulation types, we perform an ablation over all combinations of scaling, shifting, and rotation (see \Cref{sec:modulation}), using the \DIT{S}{2} model. Results are shown in \Cref{tab:mod}, with full details in \Cref{tab:mod-full-results}. Scaling provides the largest standalone benefit, significantly reducing \ac{fid}. Combining scaling with either shifting or rotation improves performance further, with the best result achieved by scaling and shifting. Interestingly, using all three modulations does not yield further gains, suggesting potential interference between them. Note that rotation modulation uses half the parameters of scaling or shifting, which may explain its comparatively lower impact.

\section{Discussion} \label{sec:discussion}
Our results show that magnitude-preserving techniques can effectively be applied to \ac{dit} models, both theoretically and empirically. However, we observed increased training time when using all proposed techniques. Future work could compare models under equal training durations to fairly assess this trade-off. A further limitation is the magnitude-preserving SiLU, which cannot be naturally extended to arbitrary magnitudes. Future work could experiment with other activation functions, such as Leaky ReLU. For this, we provide a magnitude-preserving version in \Cref{lem:lrelu}.

In conclusion, we have shown that magnitude preservation and weight growth control have a positive impact on the training dynamics in \ac{dit}. Future work could investigate scaling these results up to larger models, as well as text-to-image models. Furthermore, we showed that rotation modulation provides an alternative to existing modulation techniques.

\bibliography{references}

\clearpage
\appendix
\section*{Appendix Contents}
\startcontents[sections]
\printcontents[sections]{l}{1}{\setcounter{tocdepth}{2}}
\newpage

\section{Implementation Details}

In line with \citet{karras_analyzing_2024}, we remove all biases from linear layers. To restore the capability of the model learning a bias, we concatenate ones to the channels of the model's input. Furthermore, we implement cosine attention by normalizing the queries and keys before computing their inner products. This is equivalent and computationally more efficient.

\section{Proofs} \label{sec:proofs}

In this section, we present the remaining lemmas and proofs, demonstrating that each model component preserves or upper bounds the magnitude of its input. Our results generalize \citet{karras_analyzing_2024} to arbitrary input magnitudes.

\subsection{Linear Layer}

\begin{lemma}\label{lem:linear}
    A linear layer preserves the magnitude of its input by normalizing the rows to euclidean unit. In particular, let $\mat{W} \in \mathbb{R}^{m \times n}$ and $\vec{x} \in \mathbb{R}^n$ an input vector. Assume that the input features are \iid, such that $\mathbb{E}[x_i x_j] = 0$ for all $i \neq j$ and $\mathbb{E}[x_i^2] = \sigma^2$ for all $i$. Denote by $\hat{\mat{W}}$ the row-wise normalized version of $\mat{W}$, defined as $\hat{w}_{ij} = \nicefrac{w_{ij}}{\|\vec{w}_i\|_2}$, where $\vec{w}_i$ is the $i$-th row of $\mat{W}$. Then
    \begin{equation}
        \MP{\hat{\mat{W}}\vec{x}} = \sigma.
    \end{equation}
\end{lemma}
\begin{proof}
    Define $\hat{\mat{W}}$ and $\vec{x}$ as in \Cref{lem:linear}, then
    \begin{align}
        \MP{\hat{\mat{W}}\vec{x}}^2 &= \frac{1}{m} \sum_{i=1}^{m} \E\lft[ \lft(\sum_{j=1}^{n} \hat{w}_{ij} x_j \rgt)^2 \rgt] \\
        &= \frac{1}{m} \sum_{i=1}^{m} \sum_{j=1}^{n} \hat{w}_{ij}^2 \E\lft[ x_j^2 \rgt] \\
        &= \frac{1}{m} \sum_{i=1}^{m} \frac{\sigma^2}{\|\vec{w}_i\|_2^2} \sum_{j=1}^{n} w_{ij}^2\\
        &= \frac{1}{m} \sum_{i=1}^{m} \sigma^2\\
        &= \sigma^2.
    \end{align}
    Thus, the normalized projection preserves the magnitude of the input.
\end{proof}
In practice, the normalization of $\mat{W}$ is performed in every forward pass to ensure that the rows have unit norm, maintaining the desired property of magnitude preservation. Furthermore, \Cref{lem:linear} extends to linear embeddings, which are equivalent to linear projections when inputs are one-hot encoded.

\subsection{Residual Connection}

\begin{lemma}\label{lem:res}
    A residual connection preserves the magnitude of two inputs for any magnitude by scaling its output. In particular, assume $\vec{x},\vec{y}\in\R^n$ are uncorrelated, \ie, $\E[x_i y_j] = 0, \forall i,j\in[n]$. Further, assume $\MP{\vec{x}} = \sigma$ and $\MP{\vec{y}} = \tau$. Let $\alpha \in [0,1]$, then
    \begin{equation}
        \MP{\sqrt{\alpha} \vec{x} + \sqrt{1-\alpha} \vec{y}}^2 = \alpha \sigma^2 + (1-\alpha) \tau^2.
    \end{equation}
    As a special case, if $\MP{\vec{x}} = \MP{\vec{y}} = \sigma$, then
    \begin{equation}
        \MP{\sqrt{\alpha} \vec{x} + \sqrt{1-\alpha}\vec{y}} = \sigma.
    \end{equation}
\end{lemma}

\begin{proof}
    Define $\vec{x}$, $\vec{y}$, and $\alpha$ as in \Cref{lem:res}. Then,
    \begin{align}
        & \MP{\sqrt{\alpha} \vec{x} + \sqrt{1-\alpha} \vec{y}}^2 \\
        & = \frac{1}{n} \sum_{i=1}^n \E \lft[ \lft( \sqrt{\alpha} x_i + \sqrt{1-\alpha} y_i \rgt)^2 \rgt] \\
        & = \frac{1}{n} \sum_{i=1}^n \E \lft[ \alpha x_i^2 + (1-\alpha) y_i^2 + \sqrt{\alpha}\sqrt{1-\alpha} x_i y_i \rgt] \\
        & = \frac{\alpha}{n} \sum_{i=1}^n \E[x_i^2] + \frac{1-\alpha}{d} \sum_{i=1}^n \E[y_i^2] \\
        & = \alpha \MP{\vec{x}}^2 + (1-\alpha) \MP{\vec{y}}^2 \\
        & = \alpha \sigma^2 + (1-\alpha) \tau^2.
    \end{align}
    The special case follows as an application of the above.
\end{proof}

\subsection{Sigmoid Linear Unit (SiLU)}

\begin{lemma} \label{lem:silu}
    The \ac{silu} activation function can preserve its input magnitude by a scaling. In particular, define $s: \R \to \R$ as
    \begin{equation}
        s(\sigma) \defeq \frac{\sigma}{\sqrt{\E_{z \sim \mathcal{N}(0, \sigma)}[\operatorname{silu}^2(z)]}}.
    \end{equation}
    Let $\vec{x} \sim \mathcal{N}(\vec{0}, \sigma^2 \mat{I}_n)$ (note $\MP{\vec{x}} = \sigma$) then the \ac{silu} activation function scaled by $s(\sigma)$ preserves the magnitude of its input,
    \begin{equation}
        \MP{s(\sigma) \cdot \operatorname{silu}(\vec{x})} = \sigma.
    \end{equation}
\end{lemma}

\begin{proof}
    Define $\vec{x}$, $s: \R \to \R$, and $\operatorname{silu}: \R \to \R$ as in \Cref{lem:silu}. Then, 
    \begin{align}
        \MP{s(\sigma) \cdot \operatorname{silu}(x)}^2
        &= \frac{1}{n} \sum_{i=1}^{n} \E\lft[\lft(s(\sigma) \cdot \operatorname{silu}(x_i)\rgt)^2 \rgt]\\
        &= s^2(\sigma) \cdot \E_{z \sim \mathcal{N}(0,\sigma)}\lft[ \operatorname{silu}(z)^2 \rgt]\\
        &= \sigma^2.
    \end{align}
    Hence, the scaled activation function preserves its input's magnitude.
\end{proof}

In order to ease computation, we make the assumption that the input is sampled from $\mathcal{N}(\vec{0}, \mat{I}_n)$. In this case, the scaling factor $s(1)$ can be computed numerically and yields
\begin{equation}
    s(1) \approx \frac{1}{0.596}.
\end{equation}
Thus, the magnitude-preserving \ac{silu} activation function is evaluated as $\nicefrac{\operatorname{silu}(\cdot)}{0.596}$.

\subsection{Leaky Rectified Linear Unit (ReLU)}

\begin{lemma} \label{lem:lrelu}
    The Leaky ReLU activation function,
    \begin{equation}
        \operatorname{lrelu}_{\alpha}(x) \defeq \begin{cases}
            x & x \geq 0 \\
            \alpha x & x < 0,
        \end{cases}
    \end{equation}
    can preserve the input magnitude through appropriate scaling. Specifically, let $\alpha \in \R^+$ be the negative slope parameter, and let $\vec{x} \sim \mathcal{N}(\vec{0}, \sigma^2 \mat{I}_n)$. Then,
    \begin{equation}
        \MP{\lft( \sqrt{\frac{2}{\alpha^2 + 1}} \rgt) \operatorname{lrelu}_\alpha(\vec{x})} = \MP{\vec{x}}.
    \end{equation}
\end{lemma}

\begin{proof}
    Define $\vec{x}$ and $\alpha$ as in \Cref{lem:lrelu}. Then,
    \begin{align}
        \MP{\operatorname{lrelu}(\vec{x})}^2
        &= \frac{1}{n} \sum_{i=1}^{n} \E_{x_i \sim \mathcal{N}(0, \sigma^2)} \lft[ \operatorname{lrelu}_\alpha(x_i)^2 \rgt] \\
        &= \E_{x \sim \mathcal{N}(0, \sigma^2)} \lft[ \operatorname{lrelu}_\alpha(x)^2 \rgt] \\
        &= \int_{-\infty}^\infty  \frac{1}{\sqrt{2 \pi \sigma^2}} \cdot \exp\lft({-\frac{1}{2\sigma^2} x^2}\rgt) \cdot \operatorname{lrelu}_\alpha(x)^2 dx\\
        &= \frac{1}{\sqrt{2 \pi \sigma^2}} \lft( \int_{-\infty}^0 \exp\lft({-\frac{1}{2\sigma^2} x^2}\rgt) \alpha^2 x^2 dx + \int_0^\infty \exp\lft({-\frac{1}{2\sigma^2} x^2}\rgt)x^2 dx\rgt) \\
        &= \frac{1}{\sqrt{2 \pi \sigma^2}} \lft( \alpha^2 \int_0^\infty \exp\lft({-\frac{1}{2\sigma^2} x^2}\rgt) x^2 dx + \int_0^\infty \exp\lft({-\frac{1}{2\sigma^2} x^2}\rgt) x^2 dx\rgt) \\
        &= \frac{1}{\sqrt{2 \pi \sigma^2}} \cdot \lft(\alpha^2 + 1\rgt) \int_0^\infty \exp\lft({-\frac{1}{2\sigma^2} x^2}\rgt) \cdot  x^2 dx  \label{eq:gamma}\\
        &= \frac{1}{\sqrt{2 \pi \sigma^2}} \cdot (\alpha^2 + 1) \cdot \frac{\sqrt{\pi}}{4 \cdot \lft(\frac{1}{2 \sigma^2}\rgt)^{\nicefrac{3}{2}}} \\
        &= \lft( \frac{\alpha^2 + 1}{2} \rgt) \sigma^2,
    \end{align}
    where in \Cref{eq:gamma} we use the fact that the integral is a scaled gamma function.
\end{proof}

As a special case, setting $\alpha = 0$ yields the standard ReLU function. In this case, the following corollary holds:
\begin{corollary}\label{col:relu}
    The ReLU activation function preserves the input magnitude under appropriate scaling. Specifically, if $\vec{x} \sim \mathcal{N}(\vec{0}, \sigma^2 \mat{I}_n)$, then
    \begin{align*}
        \MP{\sqrt{2} \cdot \operatorname{relu}(\vec{x})} = \MP{\vec{x}}.
    \end{align*}
\end{corollary}

\subsection{Attention}

\begin{proof}[Proof of \Cref{lem:att}]
    Define $\mat{A} \in \R^{T \times T}$ and $\mat{V} \in \R^{T \times n}$ as in \Cref{lem:att}. We need to show that the attention mechanism preserves the magnitude of the values. Attention first normalizes the attention map,
    \begin{equation}
        \mat{B} = \operatorname{softmax}_{\beta}(\mat{A}),
    \end{equation}
    where $\operatorname{softmax}$ is defined as in \Cref{lem:att}. Now, the rows of $\mat{B}$ are normalized to be distributions over values, \ie, $\vec{b}_i \in \Delta^{T-1}$, where $\Delta^{m-1}$ is the $m$-dimensional probability simplex.
    
    Here, the rows are normalized to proper distributions over values. The output is computed by
    \begin{equation}
        \operatorname{att}(\mat{A},\mat{V}) = \mat{B}\mat{V}, \quad \operatorname{att}(\mat{A},\mat{V})_i = \sum_{t=1}^T b_{it} \vec{v}_t.
    \end{equation}
    We can now show that the output magnitude is upper bounded by the input magnitude $\sigma$,
    \begin{align}
        \MP{\operatorname{att}(\mat{A},\mat{V})_i}^2
        & = \frac{1}{n} \sum_{j=1}^{d} \E \lft[ \operatorname{att}(\mat{A},\mat{V})_{ij}^2 \rgt] \\
        & = \frac{1}{n} \sum_{j=1}^{d} \E \lft[ \lft( \sum_{t=1}^T b_{it} v_{tj} \rgt)^2 \rgt] \\
        & \leq \frac{1}{n} \sum_{j=1}^{n} \E \lft[ \sum_{t=1}^T b_{it} v_{tj}^2 \rgt] \label{eq:jensens_inequality} \\
        & = \sum_{t=1}^T b_{it} \cdot \frac{1}{n} \sum_{j=1}^{n} \E \lft[ v_{tj}^2 \rgt] \\
        & = \sum_{t=1}^T b_{it} \sigma^2 \\
        & = \sigma^2,
    \end{align}
    where Jensen's inequality is applied in \Cref{eq:jensens_inequality}, because $f(x) = x^2$ is a convex function and $\vec{b}_i \in \Delta^{T-1}$. If $b_{it} = 1$ for some $t \in [T]$ (and hence the other timesteps have weight $0$), it is trivial to see that the inequality becomes equality.
\end{proof}

\subsection{Rotation Modulation}

\begin{proof}[Proof of \Cref{lem:rot_mod}]
    Define $\mat{R}$ and $\vec{x}$ as in \Cref{lem:rot_mod}. Then,
    \begin{align}
        \MP{\mat{R}\vec{x}}^2
        &= \frac{1}{n} \sum_{i=1}^n \E \lft[ \lft(\mat{R}\vec{x} \rgt)^2_i \rgt] \\
        &= \frac{1}{n} \E \lft[ \sum_{i=1}^n \lft(\mat{R}\vec{x} \rgt)^2_i \rgt] \\
        &= \frac{1}{n} \E \lft[ \| \mat{R}\vec{x} \|_2^2 \rgt] \\
        &= \frac{1}{n} \E \lft[ \trp{\vec{x}} \trp{\mat{R}} \mat{R} \vec{x} \rgt] \\
        &= \frac{1}{n} \E \lft[ \trp{\vec{x}} \vec{x} \rgt] \\
        &= \frac{1}{n} \E \lft[ \sum_{i=1}^n x_i^2 \rgt] \\
        &= \frac{1}{n} \sum_{i=1}^n \E \lft[ x_i^2 \rgt] \\
        &= \MP{\vec{x}}.
    \end{align}
    This concludes the proof.
\end{proof}

\section{Configurations and Reproducibility}\label{sec:configurations}
A complete list of training hyperparameters is included in \Cref{tab:hyperparameters}, and detailed configurations for each setup in the magnitude preservation ablation study are summarized in \Cref{tab:dit_attributes}. Estimated training run-times for all evaluated models are reported in \Cref{tab:runtime}. For the conditioning ablation study, all corresponding hyperparameters and results are documented in \Cref{tab:mod-full-results}.

\noindent\begin{minipage}[t]{0.475\textwidth}
\centering
\captionof{table}{Configuration Attributes. A check (\cmark) indicates that a particular training attribute is enabled.}
\label{tab:dit_attributes}
\begin{tabular}{@{}lccccc@{}}
\toprule
\textbf{Attribute} & A & B & C & D & E \\
\midrule
Cosine attention         & \xmark & \cmark & \cmark & \cmark & \cmark \\
Weight norm     & \xmark & \xmark & \cmark & \cmark & \cmark \\
MP embedding             & \xmark & \xmark & \cmark & \cmark & \cmark \\
MP pos enc   & \xmark & \xmark & \cmark & \cmark & \cmark \\
MP residual              & \xmark & \xmark & \cmark & \cmark & \cmark \\
MP SiLU                  & \xmark & \xmark & \cmark & \cmark & \cmark \\
Forced weight norm       & \xmark & \xmark & \xmark & \cmark & \cmark \\
No layer norm            & \xmark & \xmark & \xmark & \xmark & \cmark \\
\bottomrule
\end{tabular}
\end{minipage}%
\hfill
\begin{minipage}[t]{0.475\textwidth}
\centering
\captionof{table}{Runtime Overview, where each model ran with batch size 256 under Config~E on a GTX~1080~Ti.}
\label{tab:runtime}
\begin{tabular}{@{}lS[table-format=3.1]S[table-format=2.1]S[table-format=2.1]@{}}
\toprule
\textbf{Model} & \multicolumn{1}{c}{Params (M)} & \multicolumn{1}{c}{Steps/sec} & \multicolumn{1}{c}{Time (h)} \\
\midrule
\DIT{B}{4}        & 131.4 & 2.13 & 52.2 \\
\DIT{B}{8}        & 131.9 & 4.11 & 27.0 \\
\DIT{S}{2}        & 33.2  & 2.11 & 52.7 \\
\textbf{\DIT{S}{4}}   & 33.3  & 5.65 & 19.7 \\
\DIT{S}{8}        & 33.5  & 7.91 & 14.0 \\
\textbf{\DIT{XS}{2}}  & 7.9   & 6.74 & 16.5 \\
\DIT{XS}{4}       & 7.9   & 11.6 & 9.6 \\
\DIT{XS}{8}       & 8.0   & 14.9 & 7.5 \\
\bottomrule
\end{tabular}
\end{minipage}

\begin{table*}[thb]
    \centering \small
    \renewcommand{\arraystretch}{1.05}
    \caption{Hyperparameters for ablation study. Each row summarizes training settings for \DIT{XS}{2} (top) and \DIT{S}{4} (bottom) across configurations A to E.}
    \par\medskip
    \begin{tabular}{@{}c<{\enspace}@{}lcccccccc@{}}
    \toprule
    & \textbf{Config} 
    & \shortstack{\#EMA\\Snapshots} 
    & \shortstack{EMA\\Stdevs} 
    & \shortstack{Warm-up\\Steps} 
    & \shortstack{LR Decay\\Start Step} 
    & \shortstack{\#Params\\(M)} 
    & \shortstack{Train\\Steps (K)} 
    & \shortstack{Batch\\Size} 
    & \shortstack{Learning\\Rate} \\ 
    \midrule
    \multirow{5}{*}{\rotatebox{90}{\textbf{\DIT{XS}{2}}}}
    & Config A & 250 & $\{0.05, 0.1\}$ & 2666 & 40000 & 7.60 & 400 & 256 & $1 \times 10^{-4}$ \\
    & Config B & 250 & $\{0.05, 0.1\}$ & 2666 & 40000 & 7.60 & 400 & 256 & $1 \times 10^{-4}$ \\
    & Config C & 250 & $\{0.05, 0.1\}$ & 2666 & 40000 & 7.86 & 400 & 256 & $1 \times 10^{-2}$ \\
    & Config D & 250 & $\{0.05, 0.1\}$ & 2666 & 40000 & 7.86 & 400 & 256 & $1 \times 10^{-2}$ \\
    & Config E & 250 & $\{0.05, 0.1\}$ & 2666 & 40000 & 7.86 & 400 & 256 & $1 \times 10^{-2}$ \\
    \midrule
    \multirow{5}{*}{\rotatebox{90}{\textbf{\DIT{S}{4}}}}
    & Config A & 250 & $\{0.05, 0.1\}$ & 2666 & 40000 & 32.89 & 400 & 256 & $1 \times 10^{-4}$ \\
    & Config B & 250 & $\{0.05, 0.1\}$ & 2666 & 40000 & 32.89 & 400 & 256 & $1 \times 10^{-4}$ \\
    & Config C & 250 & $\{0.05, 0.1\}$ & 2666 & 40000 & 33.28 & 400 & 256 & $1 \times 10^{-2}$ \\
    & Config D & 250 & $\{0.05, 0.1\}$ & 2666 & 40000 & 33.28 & 400 & 256 & $1 \times 10^{-2}$ \\
    & Config E & 250 & $\{0.05, 0.1\}$ & 2666 & 40000 & 33.28 & 400 & 256 & $1 \times 10^{-2}$ \\
    \bottomrule
    \end{tabular}
    \label{tab:hyperparameters}
\end{table*}

\begin{table*}[thb]
    \centering 
    \small
    \renewcommand{\arraystretch}{1.05}
    \caption{Hyperparameters and results for conditioning ablation study. Each row summarizes the training settings for \DIT{S}{2}. Across all runs, we used 250 EMA snapshots with standard deviations \{0.05, 0.1\}, 2666 warm-up steps, learning rate decay starting at 40K steps, and a total of 400K training steps. The batch size was set to 256, and the learning rate was fixed at $1 \times 10^{-4}$.}
    \par\medskip
    \begin{tabular}{
    @{}
    c<{\enspace}
    @{}
    c
    c
    c
    S[table-format=2.1]
    S[table-format=1.2]
    S[table-format=2.2]
    S[table-format=2.2]
    S[table-format=2.2]
    S[table-format=1.3]
    S[table-format=1.3]
    @{}}
    \toprule
    & \textbf{Scale}
    & \textbf{Shift}
    & \textbf{Rotate}
    & \textbf{Params (M)}
    & \textbf{Steps/s}
    & \textbf{FID $\downarrow$}
    & \textbf{sFID $\downarrow$}
    & \textbf{IS $\downarrow$}
    & \textbf{Precision $\uparrow$}
    & \textbf{Recall $\uparrow$} \\
    \midrule
    \multirow{7}{*}{\rotatebox{90}{\textbf{\DIT{S}{2}}}}
    & \cmark & \xmark & \xmark & 29.1 & 2.43 & 72.03 & 13.38 & 43.77 & 0.185 & 0.394 \\
    & \xmark & \cmark & \xmark & 25.6 & 2.62 & 85.23 & 73.87 & 11.12 & 0.145 & 0.387 \\
    & \xmark & \xmark & \cmark & 25.6 & 2.22 & 84.62 & 73.45 & 11.33 & 0.147 & 0.379 \\
    & \cmark & \cmark & \xmark & 32.8 & 2.34 & 69.28 & 47.17 & 14.21 & 0.186 & 0.426 \\
    & \cmark & \xmark & \cmark & 31.0 & 2.12 & 70.86 & 50.78 & 13.67 & 0.191 & 0.405 \\
    & \xmark & \cmark & \cmark & 27.4 & 2.21 & 74.01 & 57.60 & 12.82 & 0.177 & 0.397 \\
    & \cmark & \cmark & \cmark & 34.6 & 2.00 & 72.19 & 51.87 & 14.23 & 0.180 & 0.423 \\
    \bottomrule
    \end{tabular}
    \label{tab:mod-full-results}
\end{table*}

\section{Power-Function Based EMA}\label{sec:ema}

\Ac{ema} has proven to be highly effective in image generation \citep{song2020score}, but their performance is very sensitive to the decay parameter \citep{nichol_improved_2021}. Let $\var(t)$ denote the model parameters at step $t$, and $\emaG{\beta}(t)$ their \ac{ema}. The traditional \ac{ema} fixes a decaying parameter $\beta \in [0,1]$ and updates as
\begin{equation}
    \emaG{\beta}(t) = \beta \emaG{\beta}(t-1) + \lft(1 - \beta \rgt) \var(t).
\end{equation}
However, this places non-negligible weight on the random initialization of the first few training steps. Ideally, the decay rate should be small initially to reduce noise and grow larger to smooth the final parameters.

To address the aforementioned issues, \citet{karras_analyzing_2024} proposes using a power-function based \ac{ema}. We present our own derivation, closely aligned with its practical implementation and simpler to interpret,
\begin{equation}
    \ema(t) = \frac{1}{Z(t)} \sum_{\tau = 0}^t \tau^\gamma \var(\tau), \quad Z(t) = \sum_{\tau = 0}^{t} \tau^\gamma \label{eq:ema}
\end{equation}
where $Z(t)$ ensures the weights sum to $1$, and $\gamma$ is a hyperparameter that controls the overall decay, so that large $\gamma$ places more weight on recent steps. To make $\gamma$ more intuitive, \citet{karras_analyzing_2024} parametrize it via its relative standard deviation $\sigrel = \lft(\gamma + 1\rgt)^\frac{1}{2}\lft(\gamma + 2\rgt)^{-1} \lft( \gamma + 3 \rgt)^{-\frac{1}{2}}$, which gauges the ``width'' of the weighting distribution relative to the training time. Visualizations of different decay parameters are presented in \Cref{fig:ema_decay}.

In the following, we will prove closed-form update formula for the power function based \ac{ema}. For this, let $\gamma$ be fixed, and $\ema(t)$ to denote the \ac{ema} and $\var(t)$ the parameters at step $t$. Then the definition of the \ac{ema} is equivalent to the following,
\begin{align}
    \ema(t) &\defeq \frac{1}{Z(t)} \sum_{\tau=0}^t \tau^\gamma \var(\tau) \\
    &= \frac{Z(t-1)}{Z(t)}\frac{1}{Z(t-1)} \sum_{\tau=0}^{t-1} \tau^\gamma  \var(\tau) + \frac{t^\gamma}{Z(t)} \var(t)\\
    &= \frac{Z(t-1)}{Z(t)} \ema(t-1) + \frac{Z(t) - Z(t-1)}{Z(t)} \var(t)\\
    &= \frac{Z(t-1)}{Z(t)} \ema(t-1) + \lft(1 - \frac{Z(t-1)}{Z(t)} \rgt) \var(t) 
\end{align}
Unlike standard \ac{ema}, the decay factor $\nicefrac{Z(t-1)}{Z(t)}$ now depends on $t$, making it small early in training and close to to $1$ for large $t$.

However, calculating the decaying term directly can be numerically unstable, as the complexity scales linearly with the step $t$, as its just a sum over powers
\begin{equation}
    Z(t) = \sum_{\tau = 0}^t \tau^\gamma.
\end{equation}
To mitigate this, we approximate each sum via a continuous integral, which consequently can be represented in closed form,
\begin{align}
  \frac{Z(t-1)}{Z(t)} 
  &\approx \frac{\int_{0}^{t-1} \tau^\gamma \mathrm{d}\tau}{\int_{0}^{t} \tau^\gamma \mathrm{d}\tau}
  = \frac{\frac{1}{\gamma + 1}(t-1)^{\gamma + 1}}{\frac{1}{\gamma + 1} t^{\gamma + 1}}
  = \lft(1 - \frac{1}{t} \rgt)^{\gamma + 1}.
\end{align}
This derivation follows the same ideas as \citet{karras_analyzing_2024}, but is presented here in a form closer to our actual implementations, and allows for easier understanding. The resulting update formula matches the one in \citet{karras_analyzing_2024} exactly.

\paragraph{Terminology}

While this approach shares similarities with traditional \ac{ema}, it does not follow a strictly ``exponential'' decay due to its dependence on the power function. Therefore, referring to it as an \ac{ema} can be misleading. However, for consistency with prior work and to emphasize its utility in improving training dynamics, we retain the term \ac{ema} to avoid unnecessary confusion.

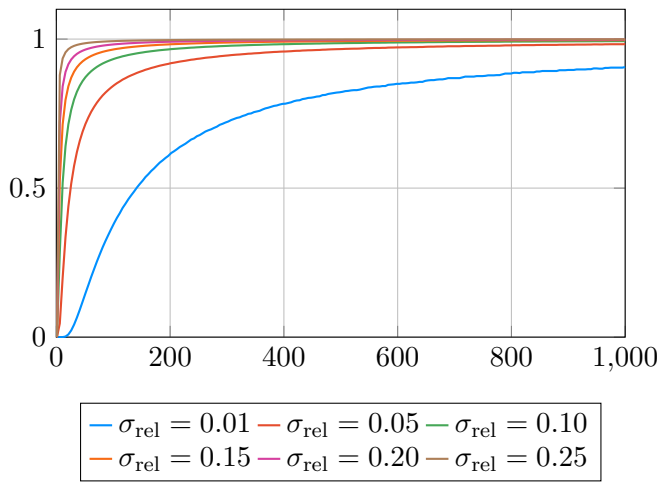
\begin{wrapfigure}{r}{0.6\linewidth}
    \centering
    \begin{tikzpicture}
    \begin{axis}[
        domain=1:1000,    
        samples=200,    
        xmin=-0.1,
        xmax=1000,
        ymin=0,
        ymax=1.1,
        width=\linewidth,
        height=0.65\linewidth,
        legend style={legend columns=3},
        legend style={at={(0.5,-0.2)}, anchor=north},
        grid={major},
    ]
    
    \pgfmathdeclarefunction{calcBeta}{2}{%
        \pgfmathparse{(1 - 1 / (#1))^(#2))}%
    }

    \addplot[C0, thick] {calcBeta(x, 96.99)};
    \addlegendentry{$\sigrel = 0.01$}

    \addplot[C1, thick] {calcBeta(x, 16.97)};
    \addlegendentry{$\sigrel = 0.05$}

    \addplot[C2, thick] {calcBeta(x, 6.93)};
    \addlegendentry{$\sigrel = 0.10$}

    \addplot[C3, thick] {calcBeta(x, 3.55)};
    \addlegendentry{$\sigrel = 0.15$}

    \addplot[C4, thick] {calcBeta(x, 1.82)};
    \addlegendentry{$\sigrel = 0.20$}
    
    \addplot[C5, thick] {calcBeta(x, 0.71)};
    \addlegendentry{$\sigrel = 0.25$}
  
\end{axis}
\end{tikzpicture}
    \caption{\Ac{ema} Decaying Factor Curves. Decaying factor $\nicefrac{Z(t-1)}{Z(t)}$ ($y$-axis) over the first 1000 steps ($x$-axis) for various relative standard deviations.}
    \label{fig:ema_decay}
    \vspace{-5mm}
\end{wrapfigure}

\paragraph{Post hoc EMA}
The choice of decay parameter strongly affects performance \citep{nichol_improved_2021}, and the same holds for our hyperparameter $\gamma$. Ideally, one could choose $\gamma$ (or equivalently $\sigrel$) \emph{after} training.
\citet{karras_analyzing_2024} show that it is possible to reconstruct different \ac{ema} profiles accurately from a single training run, by storing a small set of model snapshots under varying $\gamma$. Specifically, we store two \ac{ema} snapshots for $\gamma_1 = 16.97$ and $\gamma_2 = 6.94$ (corresponding to $\sigrel = 0.05$ and $0.1$, respectively) every \num{1600} training steps. After processing \num{400}k images. To save storage, we store each snapshot in 16-bit floating point precision. 

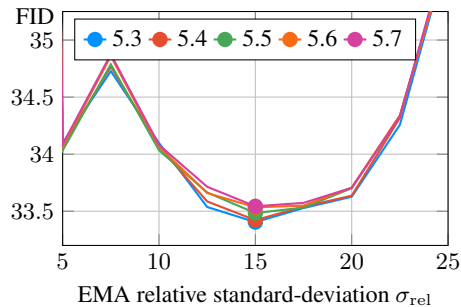
\begin{wrapfigure}{r}{0.44\linewidth}
    \vspace{-5mm}
    \begin{tikzpicture}
        \footnotesize
        \begin{axis}[
                width=\linewidth,
                height=0.7\linewidth,
                xmin=5, xmax=25,
                ymin=33.2, ymax=35.25,
                xlabel={\acs{ema} relative standard-deviation $\sigrel$},
                ylabel={\acs{fid}},
                ylabel style={
                        at={(axis description cs:0,0.98)},
                        anchor=east,
                        rotate=270,
                    },
                grid={major},
                legend pos={north west},
                legend style={legend columns=-1}
            ]
            \AddMinimumPlot[C0]{5.3}{e_xs_2-guidance5.3---ema-fid10k}
            \AddMinimumPlot[C1]{5.4}{e_xs_2-guidance5.4---ema-fid10k}
            \AddMinimumPlot[C2]{5.5}{e_xs_2-guidance5.5---ema-fid10k}
            \AddMinimumPlot[C3]{5.6}{e_xs_2-guidance5.6---ema-fid10k}
            \AddMinimumPlot[C4]{5.7}{e_xs_2-guidance5.7---ema-fid10k}
        \end{axis}
    \end{tikzpicture}
    \caption{FID-10K across varying $\sigrel$ values. Experiments were conducted using \DIT{XS}{2} using Config~E.}
    \label{fig:ema_reconstruct}
\end{wrapfigure}

For an arbitrary $\gamma^*$ value, we can approximate $\emaG{\gamma^*}$ after training by computing the least-squares fit between the two stored \ac{ema} profiles and the desired $\gamma^*$, then combine the snapshots accordingly.

\paragraph{Results}
The ability to reconstruct different \ac{ema} profiles post-training adds flexibility to this approach. In \Cref{fig:ema_reconstruct}, we calculated FID-10K scores for \DIT{XS}{2} in Config~E using various reconstructed \ac{ema} profiles. The optimal FID-10K score was achieved with $\sigrel = \SI{15}{\percent}$, improving the score by $\approx 0.64$ for all tested guidance scales. Interestingly, the optimal $\sigrel$ was not saved during training.

\section{Activation Magnitude Evolution}\label{sec:mag_evol}
To assess the impact of our magnitude-preserving techniques, we visualize the evolution of activation magnitudes across the DiT blocks of DiT-S/4 in \Cref{fig:mag_evol}.

At initialization, Config~A maintains a constant magnitude across blocks. However, after 400k training steps, it exhibits a clear trend of increasing magnitudes at the outputs of both the self-attention and MLP modules. Despite this accumulation, LayerNorm ensures that the inputs to these modules remain consistently near zero magnitude, stabilizing the model. 

In contrast, Config~E, which applies all proposed magnitude-preserving techniques and omits LayerNorm, shows a steady decline in magnitude at initialization. This behavior aligns with our theoretical result in \Cref{lem:att}, which shows that attention layers inherently reduce magnitude. After training, Config~E decrease magnitudes steadily in all blocks and avoids the excessive growth as seen in Config~A, suggesting that our techniques are effective even without normalization.

We also include Config~D, which uses the same techniques as Config~E but retains LayerNorm. At initialization, the only difference is that inputs are normalized to unit magnitude by LayerNorm. After training, its magnitude evolution closely resembles that of Config~E, with the main distinction being a spike in the MLP input due to LayerNorm. Given the similarity in both behavior and performance (see \Cref{sec:res}), we conclude that LayerNorm may not be necessary when designing DiT in a magnitude-preserving manner.

\newcommand{\MagnitudePlot}[2]{
\begin{tikzpicture}
  \footnotesize
  \pgfplotstableread[col sep=comma]{#1}\datatable
  
  \begin{axis}[
    scale only axis,
    width=0.8\textwidth,
    height=0.4\textwidth,
    scaled ticks=false,
    xlabel={DiT Blocks}, ylabel={$\MP{\vec{x}}$},
    ylabel style={
        at={(axis description cs:0,0.98)},
        anchor=east,
        rotate=270,
    },
    xtick=data,
    xticklabel style={/pgf/number format/1000 sep=},
    xmin=1, xmax=12,
    ymin=0, ymax=#2,
    grid=major,
    legend to name=mag_legend,
    legend columns=4,
    legend image post style={line width=1pt},
  ] 
    \addplot[thick, C0]
      table[x=Block,y=MSAavg]{\datatable};
    \addlegendentry{Modulated MSA Input}

    \addplot[name path=msaUp,   draw=none, forget plot]
      table[x=Block,y=MSAup]{\datatable};
    \addplot[name path=msaLow,  draw=none, forget plot]
      table[x=Block,y=MSAlow]{\datatable};
    \addplot[fill=C0, fill opacity=0.3, draw=none, forget plot]
      fill between[of=msaUp and msaLow];

    \addplot[thick, C2]
      table[x=Block,y=OUTavg]{\datatable};
    \addlegendentry{MSA Output}

    \addplot[name path=outUp,  draw=none, forget plot]
      table[x=Block,y=OUTup]{\datatable};
    \addplot[name path=outLow, draw=none, forget plot]
      table[x=Block,y=OUTlow]{\datatable};
    \addplot[fill=C2, fill opacity=0.3, draw=none, forget plot]
      fill between[of=outUp and outLow];

    \addplot[thick, C1]
      table[x=Block,y=MLPinavg]{\datatable};
    \addlegendentry{Modulated MLP Input}

    \addplot[name path=mlpInUp,  draw=none, forget plot]
      table[x=Block,y=MLPinup]{\datatable};
    \addplot[name path=mlpInLow, draw=none, forget plot]
      table[x=Block,y=MLPinlow]{\datatable};
    \addplot[fill=C1, fill opacity=0.3, draw=none, forget plot]
      fill between[of=mlpInUp and mlpInLow];

    \addplot[thick, C4]
      table[x=Block,y=MLPoutavg]{\datatable};
    \addlegendentry{MLP Output}

    \addplot[name path=mlpOutUp,  draw=none, forget plot]
      table[x=Block,y=MLPoutup]{\datatable};
    \addplot[name path=mlpOutLow, draw=none, forget plot]
      table[x=Block,y=MLPoutlow]{\datatable};
    \addplot[fill=C4, fill opacity=0.3, draw=none, forget plot]
      fill between[of=mlpOutUp and mlpOutLow];
  \end{axis}
\end{tikzpicture}
}

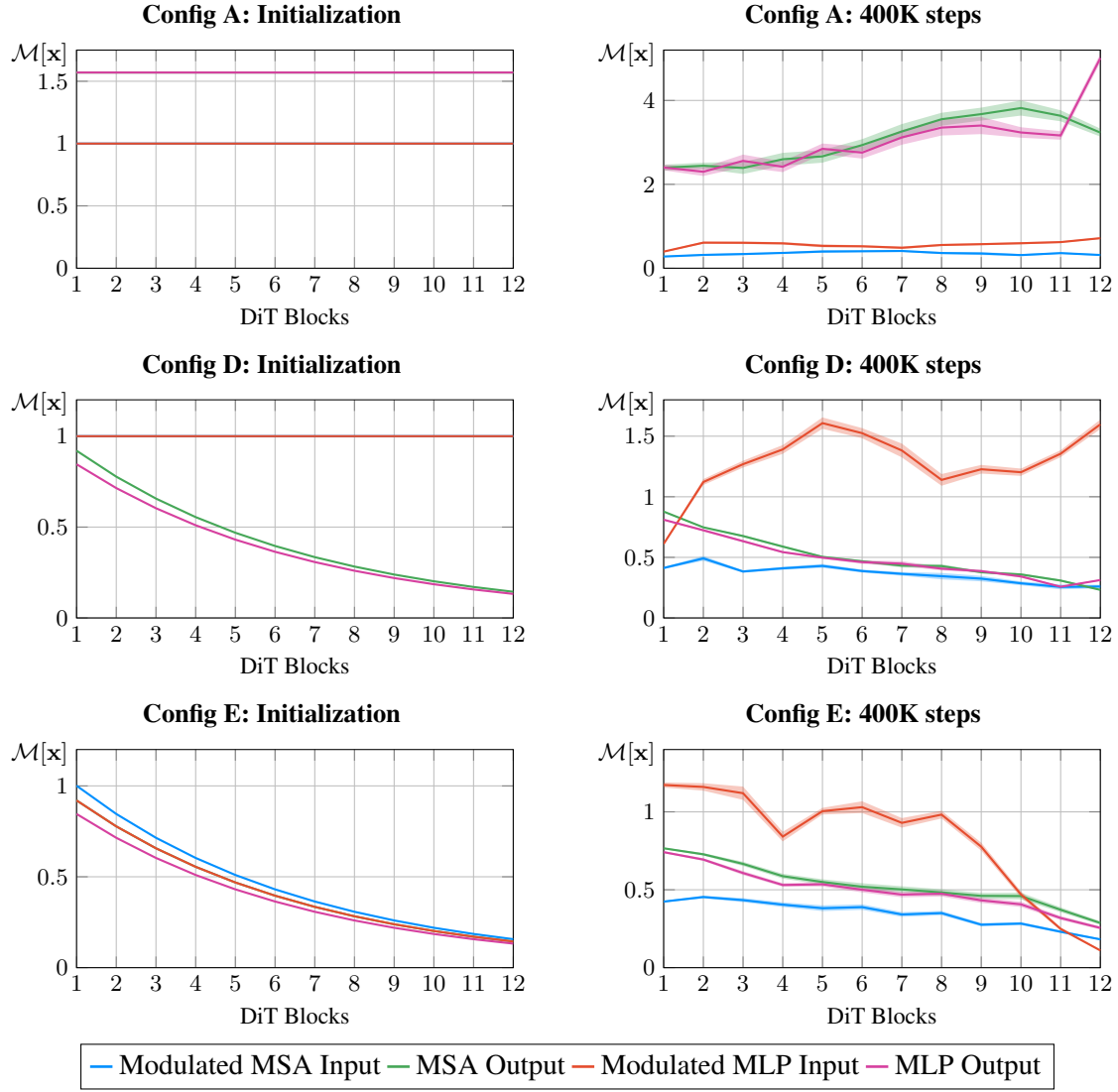
\begin{figure}[t]
    \centering
    \begin{subfigure}[t]{0.48\textwidth}
        \centering
        \captionsetup{labelformat=empty}
        \caption{\textbf{Config A: Initialization}}
        \MagnitudePlot{data/mag_A_init.csv}{1.75}
    \end{subfigure}
    \hfill
    \begin{subfigure}[t]{0.48\textwidth}
        \captionsetup{labelformat=empty}
        \caption{\textbf{Config A: 400K steps}}
        \MagnitudePlot{data/mag_A_400.csv}{5.2}
    \end{subfigure}

    \begin{subfigure}[t]{0.48\textwidth}
        \centering
        \captionsetup{labelformat=empty}
        \caption{\textbf{Config D: Initialization}}
        \MagnitudePlot{data/mag_D_init.csv}{1.2}
    \end{subfigure}
    \hfill
    \begin{subfigure}[t]{0.48\textwidth}
        \captionsetup{labelformat=empty}
        \caption{\textbf{Config D: 400K steps}}
        \MagnitudePlot{data/mag_D_400.csv}{1.8}
    \end{subfigure}

    \begin{subfigure}[t]{0.48\textwidth}
        \centering
        \captionsetup{labelformat=empty}
        \caption{\textbf{Config E: Initialization}}
        \MagnitudePlot{data/mag_E_init.csv}{1.2}
    \end{subfigure}
    \hfill
    \begin{subfigure}[t]{0.48\textwidth}
        \captionsetup{labelformat=empty}
        \caption{\textbf{Config E: 400K steps}}
        \MagnitudePlot{data/mag_E_400.csv}{1.4}
    \end{subfigure}
    \begin{subfigure}[t]{0.95\textwidth}
        \centering
        \begin{tikzpicture}
            \node at (0,0) {\pgfplotslegendfromname{mag_legend}};
        \end{tikzpicture}
    \end{subfigure}

    \caption{Activation magnitude evolution across DiT blocks in DiT-S/4. For blocks 1–12, we show mean activation magnitudes (averaged over all labels and timesteps), with shaded areas representing $\pm3$ standard deviations. Plotted are the AdaLN-modulated input (see \Cref{sec:modulation}) and the output after the residual connection, for both the self-attention and MLP modules. Results are shown at initialization (left) and after 400K training steps (right) for each configuration.}
    \label{fig:mag_evol}
\end{figure}

\section{Convergence}\label{sec:convergence}
To demonstrate the training process and verify that the choice of \num{400}K training steps was not arbitrary, we evaluated several metrics for Configurations A to E on both \DIT{XS}{2} and \DIT{S}{4}. Specifically, we used the training checkpoints saved without \ac{ema} at intervals of every \num{50}k steps. For each model and configuration combination, the metrics were calculated both in unguided mode and with guidance using a guidance scale of \num{5.0}. The metrics include
\begin{enumerate}[leftmargin=8em]
    \item[Inception Score] Measures the quality and diversity of generated images based on the output probabilities of a pretrained classifier. Higher scores indicate better generative performance \citep{salimans_improved_2016}.
    \item[Precision/Recall] Quantify the alignment between generated and real data distribution. Precision measures how much of the generated data falls within the real data manifold, while recall assesses how much of the real data is covered by the generated data.
    \item[\Acs{fid}-10K] The \acf{fid} computes the similarity between the real and generated data distributions using the mean and covariance of their feature representations. Lower values indicate higher similarity. \Acs{fid}-10K calculates this score over 10,000 samples \citep{heusel_gans_2018}.
    \item[\acs{sfid}-10K] \ac{sfid} is a variant of FID that reduces computational complexity by using a single set of real features for comparison. Like FID, lower values indicate better alignment between real and generated data distributions \citep{nash_generating_2021}.
\end{enumerate}

Plots showing the progression of these metrics over the training span are provided. For \DIT{XS}{2}, unguided and guided results can be found in \Cref{fig:convergence-xs}. Similarly, the results for \DIT{S}{4} are presented in \Cref{fig:convergence-s}.

\NewDocumentCommand\AddPlot{O{black}mm}{%
    \addplot [thick, #1] coordinates { \rawdata{#3} };
    \ifblank{#2}{}{\addlegendentry{#2};}
}

\NewDocumentCommand\MaxConvergence{mmm}{%
    \begin{subfigure}[t]{0.45\linewidth}
        \captionsetup{labelformat=empty}
        \caption{\textbf{#3}}
        \begin{tikzpicture}
            \footnotesize
            \begin{axis}[
                    scale only axis,
                    width=0.9\textwidth,
                    height=0.4\linewidth,
                    scaled ticks=false,
                    xtick={50000, 100000, 150000, 200000, 250000, 300000, 350000, 400000},
                    xticklabels={50k, 100k, 150k, 200k, 250k, 300k, 350k, 400k},
                    xmin=50000, xmax=400000,
                    xlabel={Training steps},
                    ylabel={#2},
                    ylabel style={
                            at={(axis description cs:0,0.98)},
                            anchor=east,
                            rotate=270,
                        },
                    grid={major},
                    legend pos={north west},
                    legend style={legend columns=-1}
                ]
                \AddPlot[C0]{A}{a_#1}
                \AddPlot[C1]{B}{b_#1}
                \AddPlot[C2]{C}{c_#1}
                \AddPlot[C3]{D}{d_#1}
                \AddPlot[C4]{E}{e_#1}
            \end{axis}
        \end{tikzpicture}   
    \end{subfigure}%
}

\NewDocumentCommand\PRConvergence{mmm}{%
    \begin{subfigure}[t]{0.45\linewidth}
        \captionsetup{labelformat=empty}
        \caption{\textbf{#3}}
        \begin{tikzpicture}
            \footnotesize
            \begin{axis}[
                    scale only axis,
                    width=0.9\textwidth,
                    height=0.4\linewidth,
                    scaled ticks=false,
                    xtick={50000, 100000, 150000, 200000, 250000, 300000, 350000, 400000},
                    xticklabels={50k, 100k, 150k, 200k, 250k, 300k, 350k, 400k},
                    xmin=50000, xmax=400000,
                    xlabel={Training steps},
                    ylabel={#2},
                    ylabel style={
                            at={(axis description cs:0,0.98)},
                            anchor=east,
                            rotate=270,
                        },
                    grid={major},
                    legend pos={north west},
                    legend style={legend columns=-1}
                ]
                \AddPlot[C0]{A}{a_#1-precision}
                \AddPlot[C1]{B}{b_#1-precision}
                \AddPlot[C2]{C}{c_#1-precision}
                \AddPlot[C3]{D}{d_#1-precision}
                \AddPlot[C4]{E}{e_#1-precision}
                \AddPlot[C0,densely dashed]{}{a_#1-recall}
                \AddPlot[C1,densely dashed]{}{b_#1-recall}
                \AddPlot[C2,densely dashed]{}{c_#1-recall}
                \AddPlot[C3,densely dashed]{}{d_#1-recall}
                \AddPlot[C4,densely dashed]{}{e_#1-recall}
            \end{axis}
        \end{tikzpicture}
    \end{subfigure}%
}

\NewDocumentCommand\MinConvergence{oommm}{%
    \begin{subfigure}[t]{0.45\linewidth}
        \captionsetup{labelformat=empty}
        \caption{\textbf{#5}}
        \begin{tikzpicture}
            \footnotesize
            \begin{axis}[
                    scale only axis,
                    width=0.9\textwidth,
                    height=0.4\textwidth,
                    scaled ticks=false,
                    xtick={50000, 100000, 150000, 200000, 250000, 300000, 350000, 400000},
                    xticklabels={50k, 100k, 150k, 200k, 250k, 300k, 350k, 400k},
                    xmin=50000, xmax=400000,
                    \IfValueTF{#1}{ymin=#1,}{}
                    \IfValueTF{#2}{ymax=#2,}{}
                    xlabel={Training steps},
                    ylabel={#4},
                    ylabel style={
                            at={(axis description cs:0,0.98)},
                            anchor=east,
                            rotate=270,
                        },
                    grid={major},
                    legend pos={north east},
                    legend style={legend columns=-1}
                ]
                \AddPlot[C0]{A}{a_#3}
                \AddPlot[C1]{B}{b_#3}
                \AddPlot[C2]{C}{c_#3}
                \AddPlot[C3]{D}{d_#3}
                \AddPlot[C4]{E}{e_#3}
            \end{axis}
        \end{tikzpicture}
    \end{subfigure}%
}

\begin{figure*}
    \centering
    \MaxConvergence{xs_2-noguid---steps-is}{}{Inception score $\uparrow$} \hfill
    \PRConvergence{xs_2-noguid---steps}{}{Precision $\uparrow$ (solid) and recall $\uparrow$ (dashed)}
    \par\medskip
    \MinConvergence{xs_2-noguid---steps-fid10k}{}{FID-10K $\downarrow$} \hfill
    \MinConvergence{xs_2-noguid---steps-sfid10k}{}{sFID-10K $\downarrow$}
    \par\medskip
    \MaxConvergence{xs_2-guid---steps-is}{}{Inception score $\uparrow$} \hfill
    \PRConvergence{xs_2-guid---steps}{}{Precision $\uparrow$ (solid) and recall $\uparrow$ (dashed)}
    \par\medskip
    \MinConvergence{xs_2-guid---steps-fid10k}{}{FID-10K $\downarrow$} \hfill
    \MinConvergence{xs_2-guid---steps-sfid10k}{}{sFID-10K $\downarrow$}
    \caption{\textbf{Convergence trends of generative metrics for \DIT{XS}{2} without guidance (top) and with guidance scale \num{5.0} (bottom).} Metrics including the Inception Score, Precision, Recall, FID-10k, and sFID-10k are calculated for all configurations using a guidance scale of \num{5.0}. These evaluations are performed on model checkpoints taken every 50k training steps.}
    \label{fig:convergence-xs}
\end{figure*}

\begin{figure*}
    \centering
    \MaxConvergence{s_4-noguid---steps-is}{}{Inception score $\uparrow$} \hfill
    \PRConvergence{s_4-noguid---steps}{}{Precision $\uparrow$ (solid) and recall $\uparrow$ (dashed)}
    \par\medskip
    \MinConvergence{s_4-noguid---steps-fid10k}{}{FID-10K $\downarrow$} \hfill
    \MinConvergence{s_4-noguid---steps-sfid10k}{}{sFID-10K $\downarrow$}
    \par\medskip
    \MaxConvergence{s_4-guid---steps-is}{}{Inception score $\uparrow$} \hfill
    \PRConvergence{s_4-guid---steps}{}{Precision $\uparrow$ (solid) and recall $\uparrow$ (dashed)}
    \par\medskip
    \MinConvergence{s_4-guid---steps-fid10k}{}{FID-10K $\downarrow$} \hfill
    \MinConvergence{s_4-guid---steps-sfid10k}{}{sFID-10K $\downarrow$}
    \caption{Convergence trends of generative metrics for \DIT{S}{4} without guidance (top) and with guidance scale \num{5.0} (bottom). Metrics including the Inception Score, Precision, Recall, FID-10k, and sFID-10k are calculated for all configurations using a guidance scale of \num{5.0}. These evaluations are performed on model checkpoints taken every 50k training steps.}
    \label{fig:convergence-s}
\end{figure*}

\section{Samples}
\Cref{fig:uncurated-agaric,fig:uncurated-alp,fig:uncurated-arctic_fox,fig:uncurated-daisy,fig:uncurated-jay,fig:uncurated-macaw,fig:uncurated-saint_bernard} show samples from all trained models under the same seed. This make it possible to directly compare the generations of the different models. They were all generated with guidance scale \num{5.0} and EMA $\sigrel$ of $10\%$. They are best viewed zoomed in.

\newcommand{\FourTwo}[2]{
\begin{subfigure}[b]{0.42\textwidth}
    \centering
    \captionsetup{labelformat=empty}
    \includegraphics[width=\textwidth,keepaspectratio]{#1}
    \caption{#2}
\end{subfigure}
}

\NewDocumentCommand\UncuratedSamples{mmm}{%
\begin{figure*}[p]
    \centering
    \FourTwo{figures/uncurated/A-XS-2/#1.png}{\DIT{XS}{2} (Config A).} \hspace{5mm}
    \FourTwo{figures/uncurated/A-S-4/#1.png}{\DIT{S}{4} (Config A).}
    \par\medskip
    \FourTwo{figures/uncurated/B-XS-2/#1.png}{\DIT{XS}{2} (Config B).} \hspace{5mm}
    \FourTwo{figures/uncurated/B-S-4/#1.png}{\DIT{S}{4} (Config B).}
    \par\medskip
    \FourTwo{figures/uncurated/C-XS-2/#1.png}{\DIT{XS}{2} (Config C).} \hspace{5mm}
    \FourTwo{figures/uncurated/C-S-4/#1.png}{\DIT{S}{4} (Config C).}
    \par\medskip
    \FourTwo{figures/uncurated/D-XS-2/#1.png}{\DIT{XS}{2} (Config D).} \hspace{5mm}
    \FourTwo{figures/uncurated/D-S-4/#1.png}{\DIT{S}{4} (Config D).}
    \par\medskip
    \FourTwo{figures/uncurated/E-XS-2/#1.png}{\DIT{XS}{2} (Config E).} \hspace{5mm}
    \FourTwo{figures/uncurated/E-S-4/#1.png}{\DIT{S}{4} (Config E).}
    
    \caption{Uncurated samples of \DIT{XS}{2} and \DIT{S}{4} across all configurations. \\ \hspace{\textwidth}
    Guidance scale = $5.0$ \\ \hspace{\textwidth}
    EMA relative standard-deviation = \SI{10}{\percent}\\\hspace{\textwidth}
    Class label = ``#2'' (#3)}
    \label{fig:uncurated-#1}
\end{figure*}%
}

\UncuratedSamples{agaric}{Agaric}{992}
\UncuratedSamples{alp}{Alp}{970}
\UncuratedSamples{arctic_fox}{Arctic fox}{279}
\UncuratedSamples{daisy}{Daisy}{985}
\UncuratedSamples{jay}{Jay}{17}
\UncuratedSamples{macaw}{Macaw}{88}
\UncuratedSamples{saint_bernard}{St. Bernard}{247}

\end{document}